\documentclass[twoside,11pt]{article}

\usepackage{amsmath,amsfonts,amsthm}
\usepackage{bm}

\usepackage[preprint]{jmlr2e}
\usepackage{float}

\usepackage[T1]{fontenc}
\usepackage[utf8]{inputenc}
\usepackage{lmodern}
\usepackage{booktabs}
\usepackage{array}
\usepackage{placeins}
\usepackage{xcolor}

\hypersetup{hidelinks}
\IfFileExists{xurl.sty}{\usepackage{xurl}}{}
\DeclareMathOperator{\Cov}{Cov}

\DeclareMathOperator{\Tr}{Tr}

\jmlrheading{}{}{}{}{}{}{Jongwook Kim and Jong-Min Kim}

\ShortHeadings{TSCoNet: A Two-Stage Copula CNN--LSTM}{Kim and Kim}
\firstpageno{1}

\begin{document}

\title{TSCoNet: A Two-Stage Copula CNN--LSTM for Uncertainty-Aware
  Spatio-Temporal Forecasting}

\author{\name Jongwook Kim \email jongwook.kim@bsu.edu \\
       \addr Department of Mathematical Sciences\\
       Ball State University\\
       Muncie, IN 47306, USA
       \AND
       \name Jong-Min Kim \email jongmink@morris.umn.edu \\
       \addr Statistics Discipline, Division of Science and Mathematics\\
       University of Minnesota-Morris\\
       Morris, MN 56267, USA\\
       and EGADE Business School, Tecnol\'ogico de Monterrey\\
       Ave. Rufino Tamayo, Monterrey 66269, Mexico}

\maketitle

\begin{abstract}%
Reliable forecasting of several interrelated environmental variables---such as regional precipitation and temperature, or other correlated geophysical fields---across many locations calls for accurate predictions accompanied by trustworthy statements of their uncertainty. Modern deep-learning models forecast such variables accurately but usually report no uncertainty, and forcing them to output uncertainty through maximum likelihood tends to degrade their accuracy, especially when the variables are strongly correlated. Motivated by this tension, we develop TSCoNet, a two-stage convolutional--recurrent model coupled with a Gaussian copula that jointly forecasts multiple variables over space and time while quantifying predictive uncertainty. The method first learns accurate mean forecasts and then freezes only the mean head, continuing to refine the shared representation while a dedicated head estimates the predictive variance, yielding calibrated prediction intervals after a standard recalibration, so that uncertainty is added without sacrificing point accuracy. We study the approach on simulated non-stationary spatial fields on the sphere and on a real dataset of monthly precipitation and temperature for fifty cities over 2000--2020. The model matches the accuracy of a strong deterministic forecaster while supplying calibrated prediction intervals that the deterministic model cannot, giving a single tool that provides both accurate point forecasts and reliable uncertainty for multivariate spatio-temporal data. Against alternatives that stabilize the variance by shielding the mean pathway, it is more than twice as accurate at equal coverage on the real data.
\end{abstract}


\section{Introduction}

Forecasting how interrelated environmental variables evolve across space and time---such as regional precipitation and temperature, or other geophysical fields observed across a network of locations---underlies decisions in agriculture, water and energy management, public health, and environmental risk management more broadly. Such decisions are rarely about a single variable at a single site: they involve several interrelated quantities (for instance, precipitation together with minimum and maximum temperature) at many locations and time horizons. They also depend as much on knowing \emph{how uncertain} a forecast is as on the forecast itself---an interval that is too narrow invites overconfident decisions, while one that is too wide is uninformative. Producing such forecasts is genuinely hard: the variables move together in ways that differ from place to place, their spatial dependence is non-stationary across the globe, and each carries long-term temporal structure. What practitioners need, and what remains difficult to deliver, is a forecast of all these variables jointly---over space and time---that is both accurate and accompanied by trustworthy, well-calibrated uncertainty.

Environmental and atmospheric modeling has traditionally met these demands with classical geostatistical frameworks such as kriging and Gaussian processes. These are well understood theoretically and supply closed-form uncertainty. They are linear predictors built on a covariance structure that is specified in advance (Section~\ref{sec:related}); modern deep architectures---CNN-LSTM hybrids in particular---instead learn non-linear spatio-temporal structure directly from data, yet most are deterministic. Trained to minimize squared error, they tend to interpolate high-frequency stochastic variation rather than separate it from the underlying signal \citep{rahaman2019spectral, belkin2019reconciling}, and they return a point prediction with no measure of how far it can be trusted---a serious limitation in risk-sensitive settings such as climate and environmental monitoring.

To address these limitations, we propose TSCoNet (Two-Stage Copula Network), a two-stage copula CNN-LSTM framework for multivariate spatio-temporal forecasting with integrated uncertainty quantification. While directly optimizing a joint Negative Log-Likelihood (NLL) often leads to gradient variance explosion and a collapse of the uncertainty bounds, we resolve this with a two-stage training schedule built around the multivariate copula objective; Section~\ref{sec:related} places this schedule against the staged and variance-shielding procedures already in the literature. We theoretically and empirically demonstrate that this strategy induces \textit{Uncertainty-Aware Feature Refinement}. By freezing the mean-prediction weights after an initial stabilization phase, we force the network to channel irreducible stochastic noise into a dedicated variance head (the ``noise sink''). Consequently, the shared backbone is updated via precision-weighted NLL gradients, actively filtering out noise and sharpening structural representations. The practical upshot differs by regime: in controlled simulations the refinement yields a clear RMSE advantage, whereas on real climate data its central value is to match the deterministic baseline's point accuracy while adding calibrated predictive uncertainty at no accuracy cost.

Our key contributions are summarized as follows:
\begin{itemize}
    \item A Gaussian copula couples the $n_{\text{out}}$ forecast variables at each location, so the learned dependence object $\mathbf{R}$ is a cross-variable correlation matrix rather than a spatial correlation over sites.
    \item We show that this cross-variable correlation is what makes the joint Gaussian-copula likelihood hard to optimize: the direct-NLL gradient variance inflates with the conditioning of $\mathbf{R}$, an effect that cannot arise for univariate responses.
    \item Guided by that analysis, we formulate and empirically motivate \textit{Uncertainty-Aware Feature Refinement}, a two-stage schedule that anchors the mean under an MSE objective and then freezes only the mean head while continuing to refine the shared backbone under the copula likelihood. This resolves the accuracy-versus-calibration tension the application demands---where forcing a model to report uncertainty typically degrades its point forecasts---letting a probabilistic model reach deterministic-level point accuracy---exceeding it in simulation and staying within $0.3\%$ of it on the real data---rather than trading accuracy for calibration.
    \item We validate the model on extensive spherical-IRF simulations---where it matches the deterministic and kriging baselines (beating them at low non-stationarity, on par with them at the highest), while the direct-NLL baseline, optimizing the joint likelihood from random initialization, \emph{diverges} in every regime under the gradient-variance explosion our analysis predicts---and on a real-world dataset of fifty cities worldwide---where it matches deterministic-level point accuracy while adding calibrated uncertainty (after a standard recalibration) that the deterministic baseline lacks. We evaluate the kriging baseline's own closed-form predictive intervals rather than treating it as a point predictor, which lets the uncertainty comparison be made on the same footing as the accuracy comparison, and we separate the two axes on which the learned model differs from it: intervals that adapt across space and time, and a joint law over the outputs.
    \item Against alternatives that stabilize the variance through the loss---by shielding the mean pathway or by tempering the variance gradient--- the two-stage schedule is the more accurate forecaster at matched coverage, isolating the \emph{schedule} rather than the loss as the source of the gain.
\end{itemize}

\section{Related Work}\label{sec:related}

Spatio-temporal forecasting has been studied extensively in both statistics and machine learning. In this section, we review classical approaches, modern deep learning architectures, and multi-output dependency modeling within the context of spherical environmental fields.

\subsection{Classical Geostatistics and Spherical Geometries}
Classic methods in spatial forecasting are grounded in spatial geostatistics and state-space equations. Kriging extensions, such as spatio-temporal kriging and multi-variable co-kriging---formally equivalent to Gaussian process regression \citep{rasmussen2006gaussian}---use localized spatial-temporal variograms to interpolate and predict physical fields \citep{cressie1993statistics}. While mathematically well grounded and able to provide closed-form uncertainty intervals, these statistical methods are computationally hindered by the $\mathcal{O}(N^3)$ complexity of covariance matrix inversions, where $N$ represents the number of spatial observation points. Additionally, global-scale datasets must account for the Earth's spherical geometry. As \citet{porcu2016spatio} observed, adopting simple Euclidean metrics on spherical coordinates leads to distorted distance calculations and severe parameter bias, necessitating specialized spherical covariance kernels.

\subsection{Deep Learning for Spatio-Temporal Representation}
Over the past decade, deep learning models have emerged to bypass the scalability bottlenecks of classical geostatistics. Convolutional Neural Networks (CNNs) are widely used to encode spatial features from grid-structured datasets, whereas recurrent cells, such as LSTMs or GRUs, are deployed to capture long-term temporal dependencies \citep{shi2015convolutional, wang2017deep}. CNN-LSTM hybrids combine these paradigms, learning latent features over continuous temporal sequences of spatial fields. Despite their strong empirical performance, standard deep learning architectures suffer from a lack of interpretability and are prone to overconfidence, as they are typically optimized for mean-squared error (MSE) point predictions.

\subsection{Multi-Output Dependency Modeling}
To introduce joint multi-output dependencies in deep learning, copula functions have recently gained attention \citep{salinas2019high, drouin2022tactis}. Copula theory allows for the separate modeling of marginal distributions and the joint dependence structure \citep{nelsen2007introduction}, and recent forecasting models couple deep marginal predictors with an explicit copula dependence structure---using Gaussian copulas with recurrent networks \citep{salinas2019high} or attention-based copulas with transformers \citep{drouin2022tactis}. Staged optimization in a similar spirit appears in classical copula estimation (IFM; \citealt{joe1997multivariate}), where marginal and dependence parameters are estimated sequentially, and in the two-stage training of TACTiS-2 \citep{ashok2024tactis2}; in both, the marginals are held fixed while a separate dependence module is fitted, whereas we freeze only the mean head and continue refining the shared backbone under the multivariate variance objective. However, merging copula-based probabilistic outputs with unified CNN-LSTM architectures for high-dimensional spatial fields on spherical surfaces remains comparatively unexplored. Our work fills this gap by clarifying the distinct boundaries of spatial-temporal dynamics and cross-variable dependencies under spherical settings.

Closest to our setting is \citet{huk2026rainfall}, who also pair deep marginal models with a Gaussian copula for environmental fields. Their copula couples a single variable across spatial locations; ours couples several physical variables at each location, giving the cross-variable correlation matrix $\mathbf{R}$ whose conditioning drives the training dynamics analyzed in Section~\ref{subsec:opt_dynamics}. Their task is conditional downscaling and ours is forecasting, so the two treatments of dependence are complementary rather than competing.

Beyond copula-based multivariate modeling, a parallel line of work addresses a related but distinct problem: preserving point-forecast accuracy while adding a variance head to a probabilistic network. Approaches include stop-gradient shielding of the mean pathway \citep{stirn2023faithful}, precision-weighted ($\beta$-NLL) losses that temper the variance gradient \citep{seitzer2022pitfalls}, and sequential mean-then-variance training \citep{skafte2019reliable, sluijterman2024optimal}. These methods operate on univariate responses under a diagonal-covariance assumption, and therefore do not model---nor encounter the training dynamics induced by---cross-variable dependence. Our framework instead couples multiple outputs through a Gaussian copula and analyzes how the resulting dependence structure conditions the training gradient (Section~\ref{subsec:opt_dynamics}), an interaction these single-output methods do not address. Although designed for the univariate setting, the two most directly comparable of them---$\beta$-NLL \citep{seitzer2022pitfalls} and stop-gradient shielding of the mean pathway \citep{stirn2023faithful}---adapt naturally to our multi-output copula objective, and we evaluate them as baselines (Section~\ref{subsec:exp_setup}) to separate the contribution of our two-stage \emph{schedule} from that of a variance-stabilizing loss alone. Within spatial statistics, deep learning has likewise been integrated with geostatistical structure \citep{nag2023spatiotemporal, zhan2024neural, wikle2023statistical}; our contribution is complementary, focusing on the interaction between cross-variable dependence and gradient conditioning rather than on spatial covariance modeling.

\section{Proposed Methodology}

Unlike traditional models that treat spatial forecasting as a deterministic regression task, our method maps a historical sequence of multivariate spatial grids to a joint predictive distribution, providing both expected values and dynamic uncertainty bounds over the full spatial field.

\subsection{Problem Formulation}

Let $S = \{\mathbf{x}_1, \dots, \mathbf{x}_N\}$ denote a set of $N$ spatial locations on the sphere $\mathbb{S}^2$, where each location is defined by its latitude-longitude coordinate vector $\mathbf{x}_i = (\xi_i, \theta_i)$. At each discrete time step $t \in \{1, \dots, T\}$, we observe a set of $n_{\text{out}}$ interrelated output variables at all $N$ locations. This spatio-temporal process is represented as a tensor:
\begin{equation}
\mathbf{Y}_t \in \mathbb{R}^{N \times n_{\text{out}}}.
\end{equation}
Assuming the spatial locations are organized on an $n_{\text{lat}} \times n_{\text{lon}}$ grid (such that $N = n_{\text{lat}} \times n_{\text{lon}}$), the observation at time $t$ can be reshaped into a spatial-channel grid tensor $\mathbf{Y}_t \in \mathbb{R}^{n_{\text{lat}} \times n_{\text{lon}} \times n_{\text{out}}}$. Throughout, the subscript $t$ (or $\tau$) indexes time, $(i,j)$ indexes spatial grid location, and $k \in \{1, \dots, n_{\text{out}}\}$ indexes the output variable.

Given a temporal lookback window of length $L_{\text{in}}$, the objective is to predict the joint predictive distribution of the next spatial field $\mathbf{Y}_{t+1}$ given the historical sequence:
\begin{equation}
\mathbf{Y}_{t-L_{\text{in}}+1:t} = \Big( \mathbf{Y}_{t-L_{\text{in}}+1}, \dots, \mathbf{Y}_t \Big) \in \mathbb{R}^{L_{\text{in}} \times n_{\text{lat}} \times n_{\text{lon}} \times n_{\text{out}}}.
\end{equation}

\subsection{Probabilistic CNN-LSTM Architecture}
\label{subsec:architecture}

To model both spatial structures and temporal dynamics, we propose a hybrid network consisting of three core stages:
\begin{enumerate}
    \item \textbf{Time-Distributed Spatial Encoders:} For each time step $\tau \in \{t-L_{\text{in}}+1, \dots, t\}$, the 2D spatial grid is processed by a stack of 2D convolutional blocks (each a convolution followed by max-pooling), all wrapped in a \textit{Time-Distributed} wrapper so that identical spatial filters are applied across every historical step. The resulting per-step feature maps are then flattened into a feature vector. Let $f_{\text{CNN}}$ denote this per-step spatial encoder:
    \begin{equation}
    \mathbf{h}_{\tau} = f_{\text{CNN}}(\mathbf{Y}_{\tau}) \in \mathbb{R}^{d},
    \end{equation}
    where $d$ is the dimension of the flattened spatial feature vector at each time step.
    \item \textbf{Temporal Recurrent Processing:} The sequence of per-step feature vectors $(\mathbf{h}_{t-L_{\text{in}}+1}, \dots, \mathbf{h}_t)$ is processed by a standard (fully-connected) LSTM that returns only its final hidden state, yielding a single latent temporal representation:
    \begin{equation}
    \mathbf{z}_{t} = f_{\text{LSTM}}(\mathbf{h}_{t-L_{\text{in}}+1}, \dots, \mathbf{h}_t) \in \mathbb{R}^{d'}.
    \end{equation}
    The explicit spatial layout is \emph{not} retained at this stage: spatial structure is captured only by the per-step CNN encoder, and the LSTM operates on flattened feature vectors (the grid is reconstructed only at the output stage). This is consistent with the spatial-arrangement limitation discussed in the Limitations section.
    \item \textbf{Dual-Head Parametric Outputs:} The latent vector $\mathbf{z}_t$ is mapped, through two parallel fully-connected heads, to two full-grid parameter fields arranged over the original spatial resolution $\mathbb{R}^{n_{\text{lat}} \times n_{\text{lon}} \times n_{\text{out}}}$; no transposed-convolution upsampling is used. Following standard formulations in probabilistic deep learning \citep{nix1994estimating, kendall2017what}, rather than directly outputting point predictions $\hat{\mathbf{Y}}_{t+1}$, the two heads parameterize a per-location Gaussian:
    \begin{align}
    \hat{\boldsymbol{\mu}} &= \mathbf{W}_{\mu} \mathbf{z}_{t} + \mathbf{b}_{\mu}, \\
    \log \hat{\boldsymbol{\sigma}}^2 &= \mathbf{W}_{\sigma} \mathbf{z}_{t} + \mathbf{b}_{\sigma},
    \end{align}
    where $\mathbf{W}_{\mu}, \mathbf{W}_{\sigma} \in \mathbb{R}^{(n_{\text{lat}} n_{\text{lon}} n_{\text{out}}) \times d'}$ map the latent vector to the full grid, so that---indexing the output over the grid $\mathbb{R}^{n_{\text{lat}} \times n_{\text{lon}} \times n_{\text{out}}}$---for each location $(i,j)$ and output variable $k \in \{1,\dots,n_{\text{out}}\}$, $\hat{\mu}_{i,j,k}$ is the predicted mean and $\log \hat{\sigma}_{i,j,k}^2$ the log-variance (modeled in log-space to guarantee positivity, $\hat{\sigma}_{i,j,k}^2 > 0$).
\end{enumerate}

\subsection{Joint Copula Modeling \& Probabilistic Loss Function}
\label{subsec:copula_loss}

To resolve the discrepancy between individual univariate predictions and multi-variable correlations, we employ a parametric Gaussian copula approach \citep{nelsen2007introduction} within the backpropagation process. By Sklar's theorem \citep{sklar1959fonctions}, any multivariate cumulative distribution function can be expressed in terms of its marginal distribution functions and a copula. Let $f_k(Y_{i,j,k} \mid \hat{\mu}_{i,j,k}, \hat{\sigma}_{i,j,k})$ denote the marginal Gaussian probability density function of the $k$-th output variable at location $(i,j)$. We transform the marginal probabilities into standard normal spaces using the inverse CDF $\Phi^{-1}$, which, for the Gaussian marginals used here, reduces to the standardized residual. Let $\mathbf{w}_{i,j} = [w_{i,j,1}, \dots, w_{i,j,n_{\text{out}}}]^T$ define the transformed latent residual vector at location $(i,j)$, where $w_{i,j,k} = (Y_{i,j,k} - \hat{\mu}_{i,j,k}) / \hat{\sigma}_{i,j,k}$, and let $\mathbf{R} \in \mathbb{R}^{n_{\text{out}} \times n_{\text{out}}}$ denote the joint correlation matrix. The joint probability density $f(\mathbf{Y}_{i,j})$ is represented as the product of the individual marginal densities and the copula density:
\begin{equation}\label{eq:joint_pdf}
f(Y_{i,j,1}, \dots, Y_{i,j,n_{\text{out}}}) = \left[ \prod_{k=1}^{n_{\text{out}}} f_k(Y_{i,j,k} | \hat{\mu}_{i,j,k}, \hat{\sigma}_{i,j,k}) \right] \cdot \frac{1}{\sqrt{|\mathbf{R}|}} \exp\left( -\frac{1}{2} \mathbf{w}_{i,j}^T (\mathbf{R}^{-1} - \mathbf{I}) \mathbf{w}_{i,j} \right).
\end{equation}

The network is optimized end-to-end by minimizing the joint Negative Log-Likelihood (NLL) loss function. By applying the negative logarithm to Equation \eqref{eq:joint_pdf}, the joint empirical loss decomposes into two parts: a marginal NLL term and a cross-variable copula dependency penalty:
\begin{multline} \label{eq:joint_nll}
\mathcal{L}_{\text{joint}} = \frac{1}{n_{\text{lat}} n_{\text{lon}}} \sum_{i=1}^{n_{\text{lat}}} \sum_{j=1}^{n_{\text{lon}}} \Bigg[ \sum_{k=1}^{n_{\text{out}}} \left( \log \hat{\sigma}_{i,j,k} + \frac{(Y_{i,j,k} - \hat{\mu}_{i,j,k})^2}{2 \hat{\sigma}_{i,j,k}^2} \right) \\
+ \frac{1}{2} \log |\mathbf{R}| + \frac{1}{2} \mathbf{w}_{i,j}^T (\mathbf{R}^{-1} - \mathbf{I}) \mathbf{w}_{i,j} \Bigg].
\end{multline}

We denote the full objective in Equation~\eqref{eq:joint_nll} by $\mathcal{L}_{\text{joint}}$ (the copula loss); it is equivalent (up to additive constants) to the multivariate Gaussian negative log-likelihood, so we use $\mathcal{L}_{\text{NLL}} \equiv \mathcal{L}_{\text{joint}}$ interchangeably, with detailed proofs of this equivalence provided in Appendix~B. What the copula factorization contributes is therefore not extra flexibility under Gaussian marginals but the separation of the dependence module from the marginal specification: for any family parameterized by a conditional mean and dispersion parameters, the marginals can be replaced without altering $\mathbf{R}$ or the training schedule of Section~\ref{subsubsec:training}---a substitution we discuss but do not pursue (Section~\ref{sec:limitations}). Stage~1 instead minimizes the mean squared error $\mathcal{L}_{\text{MSE}} = \frac{1}{n_{\text{lat}} n_{\text{lon}} n_{\text{out}}} \sum_{i,j,k} (Y_{i,j,k} - \hat{\mu}_{i,j,k})^2$, which omits the variance and copula terms.

By minimizing $\mathcal{L}_{\text{joint}}$, the network dynamically learns spatial-temporal features that simultaneously optimize individual point predictions, marginal uncertainty estimations ($\hat{\sigma}$), and the shared inter-variable correlation structure ($\mathbf{R}$). To strictly ensure that $\mathbf{R}$ remains symmetric and positive semi-definite during backpropagation, it is parameterized via a Cholesky-style decomposition $\mathbf{R} = \mathbf{L}\mathbf{L}^T$, where $\mathbf{L}$ is a learnable matrix with lower-triangularity enforced at each forward pass via band masking ($\mathbf{L} \leftarrow \texttt{tril}(\mathbf{L})$). The resulting $\mathbf{R}$ is then normalized to a proper correlation matrix by $\mathbf{R}_{\text{corr}} = \mathbf{D}_R^{-1/2}\mathbf{R}\mathbf{D}_R^{-1/2}$, where $\mathbf{D}_R = \mathrm{diag}(\mathbf{R})$ (distinct from $\mathbf{D}=\mathrm{diag}(\sigma_1,\dots,\sigma_K)$ used in the decomposition $\Sigma_R=\mathbf{D}\mathbf{R}\mathbf{D}$).

Directly optimizing the highly non-convex joint NLL from random weight initializations can occasionally be computationally unstable. Therefore, we adopt a \textbf{two-stage training strategy}, applied uniformly across all experiments, with an intermediate variance-initialization step (Stage 1.5) inserted between the two main stages.

\subsubsection{Two-Stage Training Strategy with Variance Initialization}
\label{subsubsec:training}

\textbf{Stage 1 (Mean Warm-up).} The backbone $\Phi$ and mean head $W_\mu$ are jointly trained under the MSE loss $\mathcal{L}_{\text{MSE}}$ to robustly anchor the predicted mean fields ($\hat{\mu}$). The variance head $W_\sigma$ is not yet active. Early stopping on validation MSE retains the best point-prediction checkpoint.

\textbf{Stage 1.5 (Variance Head Initialization).} The backbone $\Phi$ and mean head $W_\mu$ are frozen. Only the variance head $W_\sigma$ is trained under the joint NLL loss $\mathcal{L}_{\text{joint}}$ for a short warm-up. This step supplies informative initial $\hat{\sigma}$ estimates that already reflect the data's noise structure before backbone refinement begins. Without this initialization, the precision-weighting mechanism in Stage 2 would be driven by an uninformative variance head, weakening the Uncertainty-Aware Feature Refinement mechanism.

\textbf{Stage 2 (Uncertainty-Aware Feature Refinement).} The mean head $W_\mu$ remains strictly frozen ($\nabla_{W_\mu} = 0$). The backbone $\Phi$ and variance head $W_\sigma$ are jointly fine-tuned under $\mathcal{L}_{\text{joint}}$, allowing precision-weighted gradients to refine the latent representations. Crucially, while the optimization is driven by the NLL gradient, the stopping criterion is the validation MSE rather than the NLL objective: since NLL can decrease even when RMSE degrades, tracking MSE directly and restoring the lowest-MSE checkpoint provides an empirical safeguard consistent with Theorem~\ref{thm:conditioning} below, ensuring point-forecasting accuracy is never eroded during fine-tuning.

\textbf{Numerical regularization.} Two lightweight, standard regularizers stabilize the joint NLL objective and are applied uniformly to every model that optimizes it (the proposed, direct, and heteroscedastic baselines alike), so that model comparisons are not confounded by them. (i)~A \emph{variance floor} lower-bounds the predictive standard deviation, $\hat{\sigma}\ge\sigma_{\min}$ with $\sigma_{\min}=0.05$ (equivalently a floor on $\log\hat{\sigma}^2$), applied in training and evaluation. It prevents the variance head from collapsing $\hat{\sigma}\!\to\!0$ at individual locations, a collapse that otherwise sends the precision weight $1/\hat{\sigma}^2$ and the NLL to infinity---the mechanism behind the direct baseline's divergence (Section~\ref{subsec:sim_result}). (ii)~\emph{Copula shrinkage} regularizes the learned correlation toward the identity, $\mathbf{R}\leftarrow(1-\alpha)\hat{\mathbf{R}}+\alpha\mathbf{I}$ with $\alpha=0.05$, bounding $\lambda_{\min}(\mathbf{R})$ away from zero so that $\mathbf{R}^{-1}$ and $\log|\mathbf{R}|$ stay well-conditioned when the empirical residual correlation saturates near one; the effective correlation reported for the model is this shrunk $(1-\alpha)\hat{\mathbf{R}}$.

\subsection{The Effect of Cross-Variable Correlation on Gradient Conditioning}
\label{subsec:opt_dynamics}

In this subsection, we analyze how the cross-variable correlation $\mathbf{R}$ shapes the conditioning of the training gradient---and thereby the attainable accuracy---for the proposed two-stage model (TSCoNet; denoted $M_{\text{prop}}$ in the analysis below, trained via the two-stage procedure of Section~\ref{subsubsec:training}), the direct joint NLL model ($M_{\text{direct}}$), and the deterministic baseline ($M_{\text{det}}$); the complementary question of how the model separates irreducible noise from the underlying signal is taken up in Section~\ref{subsec:noise_sink}. These three models share the same CNN-LSTM backbone but differ in their output heads and training procedure: $M_{\text{prop}}$ and $M_{\text{direct}}$ carry both a mean and a variance head, whereas $M_{\text{det}}$ produces point predictions from a mean head alone. In training, $M_{\text{prop}}$ follows the two-stage schedule of Section~\ref{subsubsec:training} (MSE warm-up, then NLL refinement with the mean head frozen); $M_{\text{direct}}$ optimizes the joint NLL directly from random initialization; and $M_{\text{det}}$ minimizes the MSE alone, producing point predictions without uncertainty. The gradient-conditioning analysis below applies to these three gradient-trained networks; the kriging baseline introduced in Section~\ref{subsec:exp_setup} is a non-gradient classical interpolator and is therefore compared empirically rather than analyzed here. We define the Data Generating Process (DGP) as $Y = f^*(X) + \epsilon$, where $\epsilon \sim \mathcal{N}(\bm{0}, \Sigma_R)$. By statistical decomposition, $\Sigma_R = \mathbf{D}\mathbf{R}\mathbf{D}$, where $\mathbf{D} = \mathrm{diag}(\sigma_1,\dots,\sigma_K)$ collects the marginal standard deviations $\sigma_i$ of the $K$ output variables and $\mathbf{R}$ is the inter-variable correlation matrix. The $\sigma_i$ are properties of the individual marginals and are independent of $\mathbf{R}$. To isolate the effect of correlation, we assume constant marginal variances ($\mathbf{D}=\mathbf{I}$), yielding $\Sigma_R \equiv \mathbf{R}$; this simplification is common to both theorems below and is relaxed to the general case $\mathbf{D} \neq \mathbf{I}$ in Proposition~\ref{prop:trace_invariance}. We note that this analysis treats the noise covariance $\Sigma_R$ (equivalently $\sigma$ and $\mathbf{R}$) as fixed and known, whereas in practice they are learned jointly with the mean. The two theorems make this precise: the asymptotic optimum is invariant to $\mathbf{R}$ (Fisher consistency; Theorem~\ref{thm:invariance}), whereas the finite-sample gradient conditioning degrades as $\mathbf{R}$ strengthens (Theorem~\ref{thm:conditioning}). They therefore characterize this idealized \emph{conditioning}; whether the advantage materializes under jointly learned parameters is an empirical question examined in Section~\ref{sec:simulation_study}.

\subsubsection{Fisher Consistency and Asymptotic Invariance to \(\mathbf{R}\)}

\begin{theorem}
\label{thm:invariance}
Under the assumption of infinite data and perfect optimization (asymptotic limit), the expected RMSE of both the deterministic model ($M_{\text{det}}$) and the joint probabilistic models ($M_{\text{direct}}, M_{\text{prop}}$) is strictly independent of the correlation matrix $\mathbf{R}$.
\end{theorem}

\begin{proof}
See Appendix~A.
\end{proof}

\subsubsection{Gradient Conditioning of Direct-NLL versus Staged Optimization}

\begin{theorem}[Finite-Sample SGD Regime]
\label{thm:conditioning}
In the finite-sample stochastic gradient descent (SGD) regime (i.e., with a fixed, finite training set and a finite number of optimization steps), the direct probabilistic model ($M_{\text{direct}}$) suffers from gradient variance explosion proportional to the condition number of $\Sigma_R^{-1}$, whereas the proposed model ($M_{\text{prop}}$) keeps the Stage-1 gradient variance bounded, yielding a better-conditioned optimization trajectory than the direct joint-NLL model.
\end{theorem}

\begin{proof}
See Appendix~A.
\end{proof}

This result concerns $M_{\text{prop}}$ relative to $M_{\text{direct}}$ and explains their empirical accuracy gap; it does not assert an advantage over the deterministic baseline $M_{\text{det}}$, which also trains on MSE and is likewise free of the precision-matrix conditioning. We also note that this result does not contradict Theorem~\ref{thm:invariance}: while the asymptotic optimum $\mu_\theta^*$ is theoretically independent of $\mathbf{R}$, the finite-sample SGD optimizer cannot reach that optimum when the gradient covariance is severely ill-conditioned.

The argument uses the Gaussian form of the objective in two places: the score $\nabla_\mu\mathcal{L}_{\text{NLL}}$ is linear in the residual, and the data-generating noise is Gaussian, so that the sandwich $\Sigma_R^{-1}\Cov(Y \mid X)\Sigma_R^{-1}$ collapses to the precision matrix $\Sigma_R^{-1}$ exactly. Both hold for the Gaussian marginals used throughout this paper. Under non-Gaussian marginals the score acquires the derivative of the marginal quantile transform and is no longer linear in the residual, so the collapse---and with it the exact identification of the gradient covariance with $\Sigma_R^{-1}$---would have to be re-derived. The qualitative mechanism, however, does not depend on that simplification: whenever the objective weights the residual by an estimated dependence structure, a poorly conditioned dependence matrix inflates the gradient it produces.

\begin{proposition}\label{prop:trace_invariance}
Under the general case $\mathbf{D} \neq \mathbf{I}$, the Stage-1 gradient covariance of $M_{\text{prop}}$ equals $\Sigma_R$, whose trace $\Tr(\Sigma_R) = \sum_{i=1}^K \sigma_i^2$ is invariant to the correlation matrix $\mathbf{R}$; hence the total gradient variance is bounded by the sum of marginal variances regardless of correlation strength.
\end{proposition}

\begin{proof}
See Appendix~A.
\end{proof}

\noindent Proposition~\ref{prop:trace_invariance} thus extends the bounded-gradient property of $M_{\text{prop}}$ to the general case $\mathbf{D} \neq \mathbf{I}$, complementing Theorem~\ref{thm:conditioning}.

\begin{remark}
What makes the direct-NLL gradient explode is how the outputs are correlated, not how noisy each one is: the penalty in Theorem~\ref{thm:conditioning} depends on the correlation matrix $\mathbf{R}$ but not on the marginal noise scales. Even with large marginal variances, uncorrelated outputs ($\mathbf{R}=\mathbf{I}$) give $\lambda_{\min}(\mathbf{R})=1$ and a perfectly conditioned precision matrix; the ill-conditioning emerges only as the correlations strengthen. For the equicorrelation structure $\mathbf{R} = (1-\rho)\mathbf{I} + \rho \mathbf{1}\mathbf{1}^\top$, the eigenvalues are $1+(K-1)\rho$ (once) and $\lambda_{\min}(\mathbf{R}) = 1-\rho$ (with multiplicity $K-1$), which vanishes only as $\rho \to 1$. In the single-output case ($K=1$), $\mathbf{R}$ reduces to a scalar and the effect disappears entirely, so it cannot arise in the univariate heteroscedastic setting of \citet{stirn2023faithful, seitzer2022pitfalls}.
\end{remark}

\begin{example}
For the correlation structure used in our experiments ($\rho_{12}=0.85$, $\rho_{13}=0.75$, $\rho_{23}=0.80$), direct computation gives $\lambda_{\min}(\mathbf{R}) \approx 0.142$, condition number $\lambda_{\max}(\mathbf{R})/\lambda_{\min}(\mathbf{R}) \approx 18.3$, and precision-matrix spectral norm $\lambda_{\max}(\mathbf{R}^{-1}) = 1/\lambda_{\min}(\mathbf{R}) \approx 7.05$. This places our setting well inside the ill-conditioned regime in which Theorem~\ref{thm:conditioning} predicts gradient-variance inflation for $M_{\text{direct}}$.
\end{example}

\subsection{Noise Separation and the Noise Sink Mechanism}
\label{subsec:noise_sink}
Whereas Section~\ref{subsec:opt_dynamics} examined how the cross-variable correlation $\mathbf{R}$ conditions the gradient, this subsection turns to how the model separates irreducible noise from the underlying signal over the course of training. We mathematically establish the ``Noise Sink'' mechanism and demonstrate how the proposed multi-head architecture induces implicit regularization to overcome the spectral interpolation of noise. Our analysis proceeds in three parts. We first show that MSE training run to convergence interpolates high-frequency noise, and that the early stopping of Stage~1 mitigates this by acting as an implicit low-pass filter (Section~\ref{subsubsec:spectral_interp}). We then justify the two design choices of Stage~2 as a \emph{complementary pair}: freezing the mean head routes the irreducible residual error into the variance head---the noise sink---while keeping the backbone trainable lets precision-weighted gradients refine the shared representation without eroding point accuracy (Sections~\ref{subsubsec:stage2_mechanics}--\ref{subsubsec:noise_sink_freeze}). Whether these mechanisms deliver their intended effect is an empirical question examined in the simulation and real-data studies of Sections~\ref{sec:simulation_study} and~\ref{sec:realdata}.

\subsubsection{Spectral Interpolation of Noise and Early Stopping}
\label{subsubsec:spectral_interp}
Left to converge, MSE training interpolates not only the underlying signal but also the high-frequency noise, so that without early stopping a deterministic model memorizes noise---the ``spectral interpolation of noise'' \citep{belkin2019reconciling}. A Neural Tangent Kernel analysis \citep{jacot2018neural} makes this precise and shows that halting Stage~1 at an appropriate time acts as an implicit low-pass filter, following the same spectral-shrinkage path as Tikhonov ($L_2$) regularization \citep{ali2019continuous}; the full derivation is given in Appendix~C. This early-stopping mechanism is not specific to $M_{\text{prop}}$: it applies to any MSE-trained model, including the deterministic baseline $M_{\text{det}}$. The distinguishing component of $M_{\text{prop}}$ is Stage~2, analyzed next.

\subsubsection{The Mechanics of Uncertainty-Aware Feature Refinement in Stage 2}
\label{subsubsec:stage2_mechanics}
With the mean head frozen, the Stage-2 backbone gradient is scaled by the inverse predicted variance $1/\sigma^2$, down-weighting noise-dominated regions and emphasizing regions of structural clarity. In Stage 2, the prediction head $W_\mu$ is strictly frozen ($\nabla_{W_\mu} = 0$), and the shared backbone $\Phi$ continues to train under the NLL loss alongside the variance head $W_\sigma$. Let $\Phi$ denote the shared CNN--LSTM backbone, whose output is the latent vector $\mathbf{z}_t$ of Section~\ref{subsec:architecture}; suppressing the temporal index and writing $X$ for the predictor window, the two heads give $\mu(X) = W_\mu \Phi(X)$ and $\log\sigma^2(X) = W_\sigma \Phi(X)$ (biases absorbed into the heads). Applying the multivariable chain rule to the NLL loss (Equation~\eqref{eq:joint_nll}) and evaluating the partial derivatives of its marginal Gaussian component gives the gradient applied to the unfrozen backbone; the derivation is given in Appendix~F:
\begin{equation} \label{eq:precision_weight}
\nabla_\Phi \mathcal{L}_{\text{NLL}} = \underbrace{-W_\mu^\top \left( \frac{Y - \mu(X)}{\sigma^2(X)} \right)}_{\text{Precision-Weighted Signal Learning}} + \underbrace{W_\sigma^\top \left( \frac{1}{2} - \frac{(Y - \mu(X))^2}{2\sigma^2(X)} \right)}_{\text{Uncertainty Calibration}}
\end{equation}
The first term in Equation \eqref{eq:precision_weight} governs how Stage 2 reshapes the backbone. Unlike the MSE gradient ($-W_\mu^\top(Y-\mu)$), it is scaled by the inverse predicted variance $1/\sigma^2(X)$, functioning as a dynamic precision-weighting mechanism. In high-noise regions ($\sigma^2 \gg 0$) the denominator grows, attenuating the gradient sent to $\Phi$ and discouraging the network from fitting noise; in regions of structural clarity ($\sigma^2 \to 0$) the gradient is amplified, focusing $\Phi$ on the deterministic signal. The second term calibrates the variance head: its sign is governed by whether the squared residual exceeds the predicted variance. Where $(Y-\mu)^2 > \sigma^2$ it drives $\sigma^2$ upward---acknowledging under-estimated uncertainty---and where $(Y-\mu)^2 < \sigma^2$ it drives $\sigma^2$ downward, with the stationary point $\sigma^2 = (Y-\mu)^2$ matching the predicted variance to the realized error. The backbone thus undergoes \textit{Uncertainty-Aware Feature Refinement}. Because Stage 1.5 (Section~\ref{subsubsec:training}) pre-trains $W_\sigma$ in isolation with the backbone frozen, $\hat{\sigma}$ already reflects the noise structure of the data when Stage 2 begins, so the precision weighting is informative from the first refinement step.

\subsubsection{The Necessity of Freezing \(W_\mu\): Routing Error into the Noise Sink}
\label{subsubsec:noise_sink_freeze}
Freezing the mean head blocks the trivial route of absorbing noise into the mean, forcing the irreducible residual into the variance head. Why is it strictly necessary to freeze the linear head $W_\mu$ while the backbone $\Phi$ is updated? If $W_\mu$ were left unfrozen during Stage 2, the model would attempt to minimize the NLL loss via the most trivial mathematical route: driving the residual $(Y-\mu)^2$ to zero by arbitrarily manipulating $W_\mu$. The network would resume the spectral interpolation of noise, and the residual would vanish, leading to the pathological collapse of uncertainty ($\sigma^2 \to 0$, or ``overconfidence''). 

Freezing $W_\mu$ imposes a \textbf{structural bottleneck}. Since the linear mapping $W_\mu$ is rigidly anchored, the model cannot take the shortcut of shifting the mean baseline to swallow high-frequency noise. Although the final prediction $\mu(X) = W_\mu \Phi(X)$ dynamically evolves as the backbone $\Phi$ refines its features, its capacity to trivially interpolate noise is severely restricted. Denied the shortcut of manipulating $W_\mu$, the model is forced to channel the irreducible stochastic errors into the variance head $W_\sigma$---the \textbf{Noise Sink}. The two Stage-2 design choices therefore act as a complementary pair: freezing $W_\mu$ routes residual error into the variance head, while keeping $\Phi$ trainable lets the precision-weighted gradient adapt the shared representation to the probabilistic objective. The intended effect is that $M_{\text{prop}}$ preserves the deterministic-level accuracy established in Stage~1 while producing a usable variance estimate, rather than trading accuracy for likelihood as a direct-NLL model does. Whether this is achieved---and the conditions under which $M_{\text{prop}}$ matches or exceeds the deterministic baseline---is an empirical question examined in Section~\ref{sec:simulation_study}. We note, however, that the noise sink is imperfect: the raw variance is not automatically calibrated but tends toward overconfidence in low-noise regimes, where the variance head can collapse $\hat{\sigma}$ toward zero at a subset of locations rather than routing all residual variability into a well-formed dispersion field (Section~\ref{subsec:noise_sensitivity}). Two mechanisms contain this---the training-time variance floor (Section~\ref{subsubsec:training}) bounds the collapse so that it degrades, rather than destroys, the likelihood (unlike the direct baseline), and a standard post-hoc recalibration then restores nominal coverage. The residual effect is therefore a \emph{localized} overconfidence at a minority of locations, not a wholesale failure of the mechanism; we quantify it in Sections~\ref{subsec:sim_result} and~\ref{subsec:noise_sensitivity}.

\section{Simulation Study}
\label{sec:simulation_study}

To rigorously test our framework's capacity to handle complex geometries and varying degrees of non-stationarity, we perform a controlled simulation study. Importantly, we formulate the data-generating process under the mathematical framework of Intrinsic Random Functions (IRFs) on the sphere \citep{Materon1973}.

\subsection{Spatially and Temporally Non-Stationary Data Generation Process}
\label{subsec:sim_setup}

To simulate realistic global environmental fields, we construct non-stationary spatial structures on a unit sphere $\mathbb{S}^2$. Let $\kappa$ control the spatial non-stationarity order.

\subsubsection{Spatial Covariance Structure via Spherical IRFs}
The degree of spatial non-stationarity in our simulated fields is governed by the order parameter $\kappa$. As $\kappa$ increases, the spatial dependencies transition from standard homogeneous patterns to highly complex, non-stationary structures. The generalized spatial covariance matrix $\mathbf{\Sigma} \in \mathbb{R}^{N \times N}$ is derived from the reproducing kernel Hilbert space (RKHS) covariance kernel $K_{\kappa}$ \citep{Materon1973, huang2019sphere}:
\begin{equation}
\Sigma_{ij} = K_{\kappa}(\mathbf{x}_i, \mathbf{x}_j) + \epsilon \, \delta_{ij},
\end{equation}
where $\epsilon$ is a small nugget parameter ensuring strict positive definiteness, $\delta_{ij}$ is the Kronecker delta, and $K_{\kappa}(\mathbf{x}, \mathbf{y})$ is defined as:
\begin{equation}
\begin{aligned}\label{eq:rkhs}
K_{\kappa}(\mathbf{x}, \mathbf{y}) &= \phi_{\kappa}(h_{xy}) - \sum_{\nu=1}^{d_{\kappa}} \left[\phi_{\kappa}(h_{x\tau_{\nu}}) p_{\nu}(\mathbf{y}) + \phi_{\kappa}(h_{y\tau_{\nu}}) p_{\nu}(\mathbf{x})\right] \\
&\quad + \sum_{\nu=1}^{d_{\kappa}} \sum_{\mu=1}^{d_{\kappa}} \phi_{\kappa}(h_{\tau_{\nu}\tau_{\mu}}) p_{\nu}(\mathbf{x}) p_{\mu}(\mathbf{y}) + \sum_{\nu=1}^{d_{\kappa}} p_{\nu}(\mathbf{x}) p_{\nu}(\mathbf{y}),
\end{aligned}
\end{equation}
where $h_{ab}$ represents the great circle distance between locations $\mathbf{a}$ and $\mathbf{b}$ on the unit sphere. Additionally, $\{p_{\nu}\}_{\nu=1}^{d_{\kappa}}$ denotes a set of linearly independent basis functions (e.g., spherical harmonics) spanning the polynomial space of degree up to $\kappa$, $d_{\kappa}$ is the dimension of this space, and $\{\tau_{\nu}\}_{\nu=1}^{d_{\kappa}}$ are a set of unisolvent reference nodes on the sphere. The Intrinsic Covariance Function (ICF) $\phi_{\kappa}(h)$ is modeled using a closed-form spherical representation \citep{BussbergShieldsHuang2025}:
\begin{equation}
\phi_{\kappa}(h) = \frac{1-r^2}{4\pi} (1-2r\cos(h)+r^2)^{-3/2} - \frac{1}{4\pi}\sum_{\ell=0}^{\kappa-1} (2\ell+1)\, r^{\ell}\, P_{\ell}(\cos h), \quad 0 \leq r < 1,
\end{equation}
where $P_{\ell}$ is the Legendre polynomial of degree $\ell$ and the empty sum at $\kappa=0$ recovers the base spherical kernel. Here $r\in[0,1)$ controls the spatial correlation range/concentration of the base kernel; we fix $r=0.8$ throughout. The non-stationarity order $\kappa$ thus enters the construction in two complementary ways: it removes the lowest $\kappa$ angular-frequency (low-degree) components from the ICF $\phi_{\kappa}$---a low-frequency truncation in the sense of \citet{BussbergShieldsHuang2025}---and it fixes the degree (up to $\kappa$) of the polynomial basis $\{p_{\nu}\}$ in the generalized covariance $K_{\kappa}$ of Eq.~\eqref{eq:rkhs}.
For the scope of this simulation study, we explicitly execute experiments across three distinct conditions: $\kappa = 0$ (representing a completely stationary field), $\kappa = 1$ (first-order non-stationarity), and $\kappa = 2$ (highly complex, second-order non-homogeneous environments). By restricting our simulation to varying $\kappa \in \{0, 1, 2\}$, we can systematically evaluate the model's robustness against differing, yet controlled, degrees of spatial complexity.

\subsubsection{Temporal Non-Stationarity via Random Walk Setup}
The spatial field is evolved temporally using a random walk to introduce temporal non-stationarity. At each discrete time step $t$ ($t = 1, \dots, T$), the spatial innovations $\mathbf{W}_{:,t}$ are drawn from a zero-mean multivariate normal distribution governed by the spatial covariance matrix $\mathbf{\Sigma}$ defined in the previous section:
\begin{align}
\mathbf{W}_{:,t} &\sim \mathcal{N}(\mathbf{0}, \mathbf{\Sigma}), \quad t=1,\dots,T, \\
\mathbf{X}_{:,1} &= \mathbf{W}_{:,1}, \\
\mathbf{X}_{:,t} &= \mathbf{X}_{:,t-1} + \mathbf{W}_{:,t}, \quad t=2,\dots,T.
\end{align}

\subsubsection{Multi-Output Copula Realization}
Finally, we simulate $n_{\text{out}} = 3$ correlated output variables at each location using a Gaussian copula with a known correlation matrix $\mathbf{R} \in \mathbb{R}^{3 \times 3}$:
\begin{align}
\mathbf{Z}_{i,t} &\sim \mathcal{N}(\mathbf{0}, \mathbf{R}) \quad \text{(drawn independently for each location } i \text{ and time } t\text{)}, \\
Y_{i,t,k} &= X_{i,t} + \eta \, Z_{i,t,k}, \quad k=1,\dots,n_\text{out}.
\end{align}
Here, $i$ indexes the spatial location ($i = 1, \dots, N$), $t$ denotes the temporal snapshot ($t = 1, \dots, T$), $k$ represents the specific output variable ($k = 1, \dots, n_{\text{out}}$), and $\eta > 0$ is a noise scaling parameter that controls the signal-to-noise ratio of the generated observations. Throughout the experiments $\mathbf{R}$ is fixed to a strongly correlated structure, with off-diagonal correlations $\rho_{12}=0.85$, $\rho_{13}=0.75$, and $\rho_{23}=0.80$; this strong cross-variable correlation is precisely what induces the precision-matrix ill-conditioning analyzed in Theorem~\ref{thm:conditioning} (Section~\ref{subsec:opt_dynamics}).

This additive formulation ($Y = X + \eta Z$) follows the classical ``signal-plus-noise'' decomposition widely used in spatio-temporal geostatistics \citep{cressie1993statistics}, with $X_{i,t}$ the generalized spatio-temporal mean field and $\eta Z_{i,t,k}$ the scaled multivariate error term. Separating the marginal spatial trend from the cross-variable dependence structure in this way aligns with the data-generating process implied by Gaussian Copula Marginal Regression \citep{masarotto2012gaussian} and related copula-based regression approaches \citep{kim2021functional}, in which individual dynamics are modeled separately and the multivariate dependencies of the residuals are governed by a copula. The correlation structure $\mathbf{R}$ is held fixed across time, while the noise realizations $\mathbf{Z}_{i,t}$ vary independently from one snapshot to the next. Decoupling the components in this way gives a testing ground on which we can verify whether the architecture disentangles the spatial dynamics from the copula-driven inter-variable correlations.

The noise scale is fixed to $\eta = 10$ for the primary experiments, chosen to balance realistic signal dominance against a noise component strong enough to expose differences between models. At this level the irreducible noise accounts for roughly $24\%$ of the total marginal variance for the (near-)stationary $\text{IRF}(0)$ and $\text{IRF}(1)$ fields, and about $14\%$ for the strongly non-stationary $\text{IRF}(2)$ field, whose larger structural-signal variance dilutes the fixed noise. Section~\ref{subsec:noise_sensitivity} varies $\eta$ systematically to assess how the relative behavior of the models changes under different signal-to-noise conditions.

\subsection{Experimental Setup \& Evaluation Metrics}
\label{subsec:exp_setup}
Using the generation protocol above, we construct a spherical grid of size $N=1600$ ($40 \times 40$) over $T=200$ time steps. The Time-Distributed encoder stacks two convolutional blocks ($16$ and $32$ filters, $3 \times 3$ kernels, ``same'' padding, ReLU, each followed by $2 \times 2$ max-pooling), and the LSTM has $128$ hidden units; the network jointly predicts the $n_{\text{out}}=3$ correlated synthetic variables following the two-stage procedure of Section~\ref{subsubsec:training}. The Stage-2 fine-tuning learning rate is $10^{-4}$ for the simulation model (versus $10^{-5}$ for the real-data model; see Section~\ref{sec:realdata}).

The $T=200$ time steps are partitioned chronologically into train/validation/test segments in a $70/15/15$ ratio. With a lookback window of $L_{\text{in}}=5$, this yields $195$ forecasting samples, of which the last $30$ (snapshots $t=171,\dots,200$) form the held-out test set $\mathcal{T}_{\text{test}}$. The model is trained on the train segment with early stopping on the validation segment (no shuffling, hence no temporal leakage), and \emph{all reported metrics (RMSE, PICP, NLL) are computed on the held-out test segment only}; normalization statistics use the training period only.

To demonstrate the benefit of the probabilistic framework and to benchmark against classical geostatistical methods, we compare our proposed two-stage model ($M_{\text{prop}}$; Section~\ref{subsubsec:training}) against three comparative baselines:
\begin{enumerate}
    \item \textbf{Deterministic CNN-LSTM Baseline} ($M_{\text{det}}$): This model shares the exact spatial-temporal architecture as our proposed model but is optimized using a standard Mean Squared Error (MSE) loss function to produce point predictions without variance estimation. For a fair comparison, $M_{\text{det}}$ is trained with the same validation-MSE early stopping and best-weights restoration used in Stage~1 of $M_{\text{prop}}$.
    \item \textbf{Spatio-Temporal Kriging Baseline}: A space-time kriging \emph{forecaster} that predicts the next field from a window of past observations (an $8$-step lookback) using the grid's true coordinates and a space-time variogram with a temporal-anisotropy axis, with a per-point persistence fallback on any solver failure. It is evaluated on the same held-out test times $\mathcal{T}_{\text{test}}$ as the neural models. It remains an $\mathcal{O}(N^3)$, linear, stationarity-assuming interpolator, serving as a classical comparison point on the non-stationary IRF fields. (Its $8$-step past window differs from the neural models' $L_{\text{in}}=5$ lookback.)
    \item \textbf{Direct Probabilistic Baseline} ($M_{\text{direct}}$): This model shares the identical dual-head architecture as our proposed model but directly optimizes the joint NLL (Equation~\eqref{eq:joint_nll}) from random initialization, without the two-stage procedure (no MSE warm-up, no variance head initialization, no mean-head freezing). This serves as an ablation study to empirically validate the necessity of our staged training strategy.
    \item \textbf{Stop-gradient Baseline}: The same dual-head architecture trained single-stage under a stop-gradient (``faithful heteroscedastic'') objective in the spirit of \citet{stirn2023faithful}, which shields the mean pathway by detaching the predictive variance from the mean's gradient. It is designed to avoid variance collapse; lacking the MSE warm-up, it is less accurate.
    \item \textbf{$\beta$-NLL Baselines}: The precision-weighted $\beta$-NLL loss \citep{seitzer2022pitfalls} with $\beta=0.5$, which tempers the $1/\hat\sigma^2$ gradient, in two variants---from random initialization (\emph{scratch}) and with the Stage-1 MSE warm-up (\emph{MSE warm-up}). The pair isolates whether the warm-up alone rescues the divergent from-scratch training.
\end{enumerate}
These three additional baselines are the multivariate-copula adaptation of the univariate heteroscedastic-regression methods discussed in Section~\ref{sec:related}; they receive the same variance floor and copula shrinkage (Section~\ref{subsubsec:training}) as the proposed and direct models, so the comparison isolates the training schedule, not the regularization.

The predictive accuracy is evaluated using the root-mean-square error (RMSE), reported both per output variable and as the overall RMSE across all $n_{\text{out}}=3$ outputs:
\begin{equation}
\text{RMSE}_\text{overall} = \sqrt{\frac{1}{N\,|\mathcal{T}_{\text{test}}|\, n_\text{out} } \sum_{i=1}^{N} \sum_{t \in \mathcal{T}_{\text{test}}} \sum_{k=1}^{n_{\text{out}}} (Y_{i,t,k} - \hat{\mu}_{i,t,k})^2},
\end{equation}
where $i$ indexes the spatial locations and $t$ ranges over the held-out test snapshots $\mathcal{T}_{\text{test}}$ ($|\mathcal{T}_{\text{test}}|=30$), $Y_{i,t,k}$ is the simulated true value, and $\hat{\mu}_{i,t,k}$ is the predicted mean for the $k$-th output variable. The per-variable $\text{RMSE}_k$ is defined analogously over output $k$ alone (Appendix~G). The PICP and NLL defined below are likewise computed over the held-out test snapshots $\mathcal{T}_{\text{test}}$ only.

To evaluate the quality of the uncertainty intervals we compute the \textbf{95\% Prediction Interval Coverage Probability (PICP)}, the empirical fraction of held-out observations falling inside $\hat{\mu}_{i,t,k} \pm 1.96\,\hat{\sigma}_{i,t,k}$ (defined in Appendix~G). A PICP close to the nominal $0.95$ indicates good calibration, whereas a substantially lower value indicates an overconfident model whose intervals are too narrow. Finally, we compute the \textbf{Average per-output marginal Negative Log-Likelihood (Avg. NLL)} to evaluate the quality of the univariate predictive distributions. Concretely, for each output $k$ we compute
\begin{equation}
\text{NLL}_k = \frac{1}{N\,|\mathcal{T}_{\text{test}}|}\sum_{i=1}^{N}\sum_{t \in \mathcal{T}_{\text{test}}}\left(\log\hat{\sigma}_{i,t,k} + \frac{(Y_{i,t,k}-\hat{\mu}_{i,t,k})^2}{2\hat{\sigma}_{i,t,k}^2}\right) + \frac{1}{2}\log(2\pi),
\end{equation}
and report the average over outputs. Note that the full \emph{joint} NLL including the copula correction term ($\frac{1}{2}\log|\mathbf{R}| + \frac{1}{2}\mathbf{w}^T(\mathbf{R}^{-1}-\mathbf{I})\mathbf{w}$) serves as the \emph{training} objective but is not separately tabulated; the per-output marginal NLL is the more directly interpretable measure of per-variable uncertainty calibration quality. PICP and NLL are not applicable to the deterministic baselines.

Because the network's raw predictive variance can be miscalibrated---in particular overconfident in high signal-to-noise regimes (Sections~\ref{subsec:noise_sink} and~\ref{subsec:noise_sensitivity})---we apply a single post-hoc recalibration \citep{kuleshov2018accurate} before reporting interval coverage. Writing $z_{i,t,k} = (Y_{i,t,k}-\hat{\mu}_{i,t,k})/\hat{\sigma}_{i,t,k}$ for the standardized residuals on the \emph{validation} window, we fit, per output variable $k$, a scalar that multiplies the predictive standard deviation. This scalar can be set by more than one criterion; we present two here. The first is the MLE (temperature) scale
\begin{equation*}
s_{\mathrm{var},k} = \sqrt{\operatorname{mean}_{i,t}\!\left(z_{i,t,k}^2\right)},
\end{equation*}
which rescales the mean predictive variance to unity and maximizes the validation Gaussian likelihood; the second is the coverage-matched scale
\begin{equation*}
s_{\mathrm{cov},k} = \operatorname{quantile}_{i,t}\!\left(|z_{i,t,k}|,\, 0.95\right)/1.96,
\end{equation*}
which forces the validation intervals $\hat{\mu}_{i,t,k}\pm 1.96\,s\,\hat{\sigma}_{i,t,k}$ to cover exactly $95\%$. The two scales coincide when the standardized residuals are Gaussian and diverge under heavier-tailed residuals. The fitted scale is applied unchanged to the test window; because only $\hat{\sigma}$ is rescaled, $\hat{\mu}$ and hence RMSE are unaffected. Because this scaling is a single scalar per output variable, it adjusts the overall magnitude of the predictive intervals but leaves the learned per-location, per-time structure of $\hat{\sigma}$ intact: the pattern of where the model is more or less certain is determined by training, not by the recalibration. This learned structure is measurable through the coefficient of variation of $\hat{\sigma}$: the across-location CV is substantial in every regime ($0.2$--$0.8$ in simulation, $1.3$ on the real data), confirming a genuine spatial uncertainty field rather than a rescaled constant, while a temporal component emerges only where the data-generating process contains one---negligible across-time variation in the (temporally homoscedastic) simulation ($\mathrm{CV}\approx0.03$) but a clear across-time signal on the real climate data ($\mathrm{CV}=0.47$). We report both scales throughout (Tables~\ref{tab:overall_metrics}, \ref{tab:noise_calibration}, and~\ref{tab:real_uq}) and read each metric under the scale it optimizes: coverage (PICP) under the coverage-matched $s_{\mathrm{cov}}$, and likelihood (NLL) under the MLE-temperature $s_{\mathrm{var}}$. When the standardized residuals are Gaussian the two coincide and the choice is immaterial (as on the real data); they diverge only under heavier-tailed residuals, where---by construction---$s_{\mathrm{cov}}$ preserves nominal coverage while $s_{\mathrm{var}}$ keeps the mean NLL from being dominated by a small set of overconfident locations (Section~\ref{subsec:sim_result}).

\subsection{Simulation Results \& Discussion}
\label{subsec:sim_result}

The overall comparative performance metrics across all spatial complexities are detailed in Table \ref{tab:overall_metrics}.

\begin{table}[!ht]
\centering
\caption{Overall predictive performance on the held-out test set (mean $\pm$ standard deviation over 50 Monte Carlo runs per IRF order, noise scale $\eta=10$). PICP and NLL are given under both post-hoc recalibration scales, coverage-matched ($s_{\mathrm{cov}}$) and MLE-temperature ($s_{\mathrm{var}}$), as $s_{\mathrm{cov}}/s_{\mathrm{var}}$. Bold marks the lowest \emph{mean} RMSE in each panel; the lower block of each panel holds the three variance-stabilization baselines of Section~\ref{subsec:exp_setup}. A dash means the model supplies no predictive distribution.}
\label{tab:overall_metrics}
\begin{tabular}{l c c c}
\toprule
\textbf{Model} & \textbf{Overall RMSE} & \textbf{PICP} ($s_{\mathrm{cov}}/s_{\mathrm{var}}$) & \textbf{NLL} ($s_{\mathrm{cov}}/s_{\mathrm{var}}$) \\
\midrule
\multicolumn{4}{c}{\textbf{Panel A: Stationary Field (IRF(0))}} \\
\midrule
\textbf{Proposed (Two-stage)}     & $\mathbf{0.9289 \pm 0.0893}$ & $0.906/0.886$ & $1.652/1.687$ \\
Direct Prob (NLL-only)            & $183.23 \pm 69.36$  & $0.945/1.000$ & $7.229/6.587$ \\
Deterministic Baseline            & $1.0668 \pm 0.0694$ & --            & --            \\
Kriging Baseline                  & $1.0454 \pm 0.1842$ & $0.951/0.952$ & $1.468/1.452$ \\
\addlinespace[3pt]
Stop-gradient                     & $1.7214 \pm 0.1057$ & $0.947/0.939$ & $2.573/2.568$ \\
$\beta$-NLL (scratch)             & $207.15 \pm 72.54$  & $0.942/1.000$ & $7.381/6.708$ \\
$\beta$-NLL (MSE warm-up)         & $7.4039 \pm 1.2697$ & $0.926/1.000$ & $3.862/3.565$ \\
\midrule
\multicolumn{4}{c}{\textbf{Panel B: 1st-Order Non-Stationary (IRF(1))}} \\
\midrule
\textbf{Proposed (Two-stage)}     & $\mathbf{0.9212 \pm 0.0690}$ & $0.902/0.883$ & $1.660/1.700$ \\
Direct Prob (NLL-only)            & $199.93 \pm 88.29$  & $0.945/1.000$ & $7.306/6.658$ \\
Deterministic Baseline            & $1.0775 \pm 0.0857$ & --            & --            \\
Kriging Baseline                  & $1.0374 \pm 0.1700$ & $0.950/0.952$ & $1.458/1.446$ \\
\addlinespace[3pt]
Stop-gradient                     & $1.7188 \pm 0.1217$ & $0.947/0.940$ & $2.599/2.594$ \\
$\beta$-NLL (scratch)             & $211.19 \pm 63.68$  & $0.942/1.000$ & $7.419/6.741$ \\
$\beta$-NLL (MSE warm-up)         & $7.2260 \pm 0.9711$ & $0.925/1.000$ & $3.849/3.554$ \\
\midrule
\multicolumn{4}{c}{\textbf{Panel C: 2nd-Order Non-Stationary (IRF(2))}} \\
\midrule
\textbf{Proposed (Two-stage)}     & $\mathbf{0.8147 \pm 0.2602}$ & $0.883/0.984$ & $144.44/3.586$ \\
Direct Prob (NLL-only)            & $284.66 \pm 25.61$  & $0.948/1.000$ & $7.753/7.066$ \\
Deterministic Baseline            & $0.9817 \pm 0.3871$ & --            & --            \\
Kriging Baseline                  & $1.2062 \pm 0.9350$ & $0.949/0.965$ & $3.168/1.414$ \\
\addlinespace[3pt]
Stop-gradient                     & $1.7494 \pm 0.3397$ & $0.949/0.964$ & $3.446/2.922$ \\
$\beta$-NLL (scratch)             & $311.39 \pm 25.62$  & $0.948/1.000$ & $7.866/7.166$ \\
$\beta$-NLL (MSE warm-up)         & $5.3810 \pm 4.8651$ & $0.924/0.965$ & $3.212/3.031$ \\
\bottomrule
\end{tabular}
\end{table}

Table~\ref{tab:overall_metrics} reveals a consistent accuracy ordering, established over 50 Monte Carlo replications and corroborated by per-seed paired comparisons. The \textbf{Direct Probabilistic Baseline}, which optimizes the joint NLL from scratch without an MSE warm-up, does not merely trade accuracy for likelihood---it \emph{diverges}, with mean RMSE of $183$--$285$ across the three IRF orders (two to three orders of magnitude above the proposed model), so TSCoNet attains a lower RMSE on $100\%$ of paired realizations. This is the extreme, empirical form of the gradient-conditioning mechanism of Theorem~\ref{thm:conditioning}. From random initialization the variance head collapses---$\hat{\sigma}$ is pinned at the floor at essentially every location ($\%$-at-floor $=100$)---so the precision weight $1/\hat{\sigma}^2$ becomes enormous and amplifies the conditional-mean gradient until training destabilizes; because the copula factor is now a genuine trainable weight, it simultaneously drives the residual correlation toward a near-singular $\mathbf{R}$ that compounds the ill-conditioning. That this reflects the \emph{training schedule} rather than an unfair configuration is shown directly by the variance-stabilization baselines in the lower block of each panel: the same objective trained from scratch fails identically (RMSE $207$--$311$), yet adding only a Stage-1 MSE warm-up recovers it to RMSE $5$--$7$, while the variance-shielding stop-gradient baseline trains stably at RMSE $\approx 1.7$ throughout. This also settles whether the numerical regularizers make the two-stage schedule redundant. The copula shrinkage ($\alpha=0.05$) and the variance floor that keep the proposed model well-conditioned are applied \emph{identically} to the direct baseline (Section~\ref{subsubsec:training}), yet it still collapses: shrinkage \emph{bounds} the ill-conditioning of Theorem~\ref{thm:conditioning} but does not remove it, and with the mean and variance heads learned jointly from random initialization the collapse occurs before a bounded $\mathbf{R}^{-1}$ can help. Only anchoring the mean under MSE first---the two-stage schedule---prevents it; the regularization is a stabilizer, not a substitute for the schedule. The \textbf{Deterministic Baseline} (MSE-only) is the far stronger competitor, with mean RMSE comparable to the proposed model. The proposed model matches it on average and, more importantly, is more robust on the difficult realizations where the deterministic baseline's error inflates---visible in the substantially larger Deterministic RMSE standard deviation at IRF(2) ($\pm 0.39$ vs.\ $\pm 0.26$), and in a deterministic ``failure tail'' on the hardest IRF(2) seeds, where the proposed model's mean error remains markedly lower.

The proposed and deterministic models attain their lowest absolute RMSE in the IRF(2) panel. We emphasize that this reflects the data-generating process---in RMSE terms the second-order non-stationary field is, for these MSE-based models, the easier regime, since its larger structural-signal variance dominates the fixed noise (Section~\ref{subsec:sim_setup})---rather than the proposed model \emph{exploiting} complexity. Kriging's mean RMSE is inflated by a few hard realizations, though its median likewise improves. The informative quantity is the \emph{relative} advantage, which is in fact largest at \emph{low} complexity: at IRF(0) and IRF(1) the proposed model achieves the lowest RMSE outright, beating the deterministic baseline on $90$--$98\%$ and the kriging baseline on $80$--$86\%$ of paired realizations. At the most complex IRF(2) regime---comparable in non-stationarity to the real-data application---the margin narrows: the proposed model still attains the lowest \emph{mean} RMSE ($0.81$), but it is essentially tied with kriging on the robust median (analyzed in detail below), and its lead over the deterministic \emph{mean} ($0.81$ vs.\ $0.98$) is far smaller than the near-uniform lead at IRF(0)/IRF(1). Its clearest advantage in this regime is robustness rather than outright dominance.

The kriging baseline is best described as regime-dependent. Under the (near-)stationary IRF(0)/IRF(1) fields its error is well-behaved (mean $\approx$ median) and the proposed model is more accurate on $80$--$86\%$ of realizations. Under IRF(2) the two are even on accuracy: a $56\%$ paired win rate, and a robust median of $0.81$ for both, the mean inflation noted above leaving kriging at $1.21$ on a handful of hard random-walk realizations. Its predictive intervals are taken up below.

On the uncertainty side we report PICP and NLL under \emph{both} recalibration scales (Table~\ref{tab:overall_metrics}), reading coverage under the coverage-matched $s_{\mathrm{cov}}$ and likelihood under the MLE-temperature $s_{\mathrm{var}}$, each under the scale it optimizes (Section~\ref{subsec:exp_setup}). The network's raw variance is overconfident at $\eta=10$; after coverage recalibration the proposed model's average PICP is $0.88$--$0.91$ across the three orders. It falls short of the nominal $0.95$ because the scale is fit on the validation window---where it attains near-nominal coverage---whereas the held-out test window has drifted under the random-walk temporal dynamics; this validation-to-test gap narrows as the noise level rises (PICP $0.93$ at $\eta=30$; Table~\ref{tab:noise_calibration}). Under either scale the recalibrated proposed model decisively outperforms the direct-NLL baseline on marginal NLL ($1.7$ vs.\ $7.2$ at IRF(0)). 

The stop-gradient baseline is the informative comparison here, and it does not favor us on coverage: it attains near-nominal PICP in every regime ($0.947$--$0.949$ under $s_{\mathrm{cov}}$) where the proposed model reaches only $0.883$--$0.906$, and at IRF(2) it is also the better-behaved on marginal NLL ($3.45/2.92$ versus $144.44/3.59$). The mechanism is visible in the variance-head diagnostics: shielding the mean pathway leaves its dispersion field intact ($\%$-at-floor $\approx 0$, median $\hat\sigma \approx 0.6$), whereas the proposed model pins about half of the locations at the floor at IRF(0)/IRF(1), so a single recalibration scalar cannot restore coverage everywhere. What that calibration costs is point accuracy: the stop-gradient RMSE is $1.72$--$1.75$ against $0.81$--$0.93$ for the proposed model, $85$--$115\%$ higher. The two therefore occupy different points on the accuracy--calibration trade-off. Our schedule is designed for the end of it where point accuracy is preserved and coverage is recovered post hoc, which is the regime the application in Section~\ref{sec:realdata} demands; a variance-shielding objective buys calibration by conceding the accuracy that motivated using a deep forecaster in the first place. We note also that the collapse driving our under-coverage is specific to these high signal-to-noise simulated fields: on the real data the variance head does not collapse ($\%$-at-floor $=9.4$) and the proposed model attains both near-nominal coverage and the best NLL (Section~\ref{subsec:real_performance}).

A caveat is visible in the two NLL columns. At IRF(0)/IRF(1) the two scales agree ($\text{NLL}\approx1.7$ either way): here the variance head leaves about half of the locations at the floor ($\%$-at-floor $\approx 50$), but the fitted $s_{\mathrm{cov}}$ ($\approx 12$--$13$) is large enough to lift them, so both scales are well-behaved. At IRF(2) the two scales diverge---$\text{NLL}=144$ under $s_{\mathrm{cov}}$ versus $3.6$ under $s_{\mathrm{var}}$. Here the bulk of the field is already near-calibrated, so the fitted $s_{\mathrm{cov}}$ (median $0.68$ across realizations) is close to one and cannot lift the small residual set of overconfident locations (only $\sim4\%$ sit at the floor, but their standardized residuals are extreme); the mean NLL, being tail-sensitive, is dominated by them. The MLE-temperature scale $s_{\mathrm{var}}$, set by the mean squared standardized residual, inflates the intervals enough to bound this tail and yields a stable $\text{NLL}\approx3.6$ (at the cost of over-coverage, $\text{PICP}=0.98$). The training-time variance floor (Section~\ref{subsubsec:training}) is what keeps this failure mode bounded rather than divergent: it caps the collapse at $\hat\sigma=\text{floor}$ instead of $\hat\sigma\!\to\!0$, so that $s_{\mathrm{var}}$ can recover a usable likelihood. This residual, localized overconfidence is the limit of a single global recalibration scalar; Section~\ref{subsec:noise_sensitivity} examines how it varies with the noise level.

Kriging is the demanding comparison on this axis, since it returns a closed-form predictive variance alongside its prediction. We recovered that variance for the same 50 realizations per regime and evaluated it identically. On these fields its intervals are the better calibrated of the two: $\text{PICP}=0.951/0.950/0.949$ under $s_{\mathrm{cov}}$ against $0.906/0.902/0.883$ for the proposed model, with a lower marginal NLL under both scales, and with fitted scales so close to one ($s_{\mathrm{cov}}\approx0.97$, $s_{\mathrm{var}}\approx1.01$ at IRF(0)/IRF(1)) that no recalibration is needed at all. The reason is specific and bounded: the simulated fields are generated as Gaussian processes, so the kriging model is close to correctly specified here and a correctly specified linear predictor obtains its predictive variance for free. The advantage is confined to that condition---on the real data, where the model is misspecified, kriging's raw intervals over-cover ($0.982$), require a substantial correction ($s_{\mathrm{cov}}=0.69$), and still reach a marginal NLL an order of magnitude worse than the proposed model's ($0.684$ against $0.032$; Table~\ref{tab:real_uq}).

Two properties limit what that calibration delivers, and both are structural. The kriging variance is a function of the data \emph{configuration}---the design geometry and the fitted variogram---and not of the observed \emph{values}, a long-recognized property of the estimator \citep{journel1986geostatistics, goovaerts1997geostatistics}, so its intervals are exactly constant in time in every regime and carry no information about when a forecast is less trustworthy. And because the baseline is three independent per-variable krigings, it prices no cross-variable dependence: modelling that dependence is worth $1.0$--$2.7$ nats per space--time point against a marginal deficit of $0.2$--$2.2$, so on the joint predictive law the ordering reverses in all three regimes. Appendix~H gives both comparisons in full.

\FloatBarrier
\subsubsection*{Visual Inspection of Temporal and Spatial Predictions}
The prediction quality is illustrated for the highest-complexity regime, IRF(2), in Figures~\ref{fig:timeseries_kappa2} and~\ref{fig:heatmap_kappa2_y1}. The time series plot (Figure~\ref{fig:timeseries_kappa2}) traces true versus predicted sequences at nine randomly chosen spatial coordinates, showing that the proposed model follows the true trajectories closely; the analogous plots for the other non-stationarity levels $\kappa$ exhibit the same behavior and are provided in Appendix~D. The spatial heatmap (Figure~\ref{fig:heatmap_kappa2_y1}) compares True, Predicted, and Residual fields for $Y_1$; the flat, randomized residual distribution indicates that the CNN-LSTM generalizes the localized spatial patterns without imposing geographical biases. Corresponding heatmaps for the remaining outputs $Y_2$ and $Y_3$, and for the other non-stationarity levels, are provided in Appendix~D.

\begin{figure}[!ht]
\centering
\includegraphics[width=0.48\linewidth]{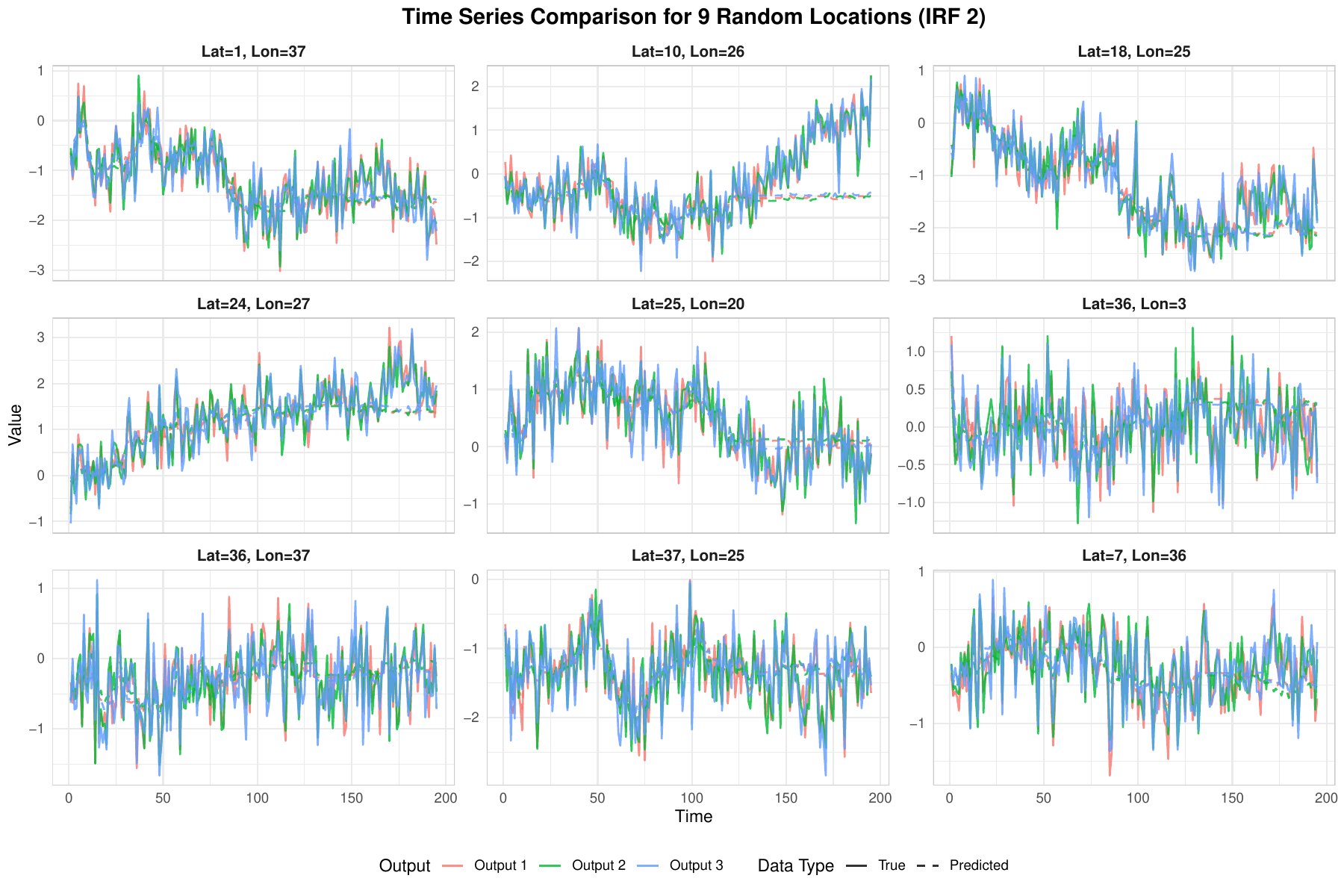}
\caption{Time series of true (solid line) and predicted (dashed line) values at nine randomly chosen grid points for IRF(2).}
\label{fig:timeseries_kappa2}
\end{figure}

\begin{figure}[!ht]
\centering
\includegraphics[width=0.64\linewidth]{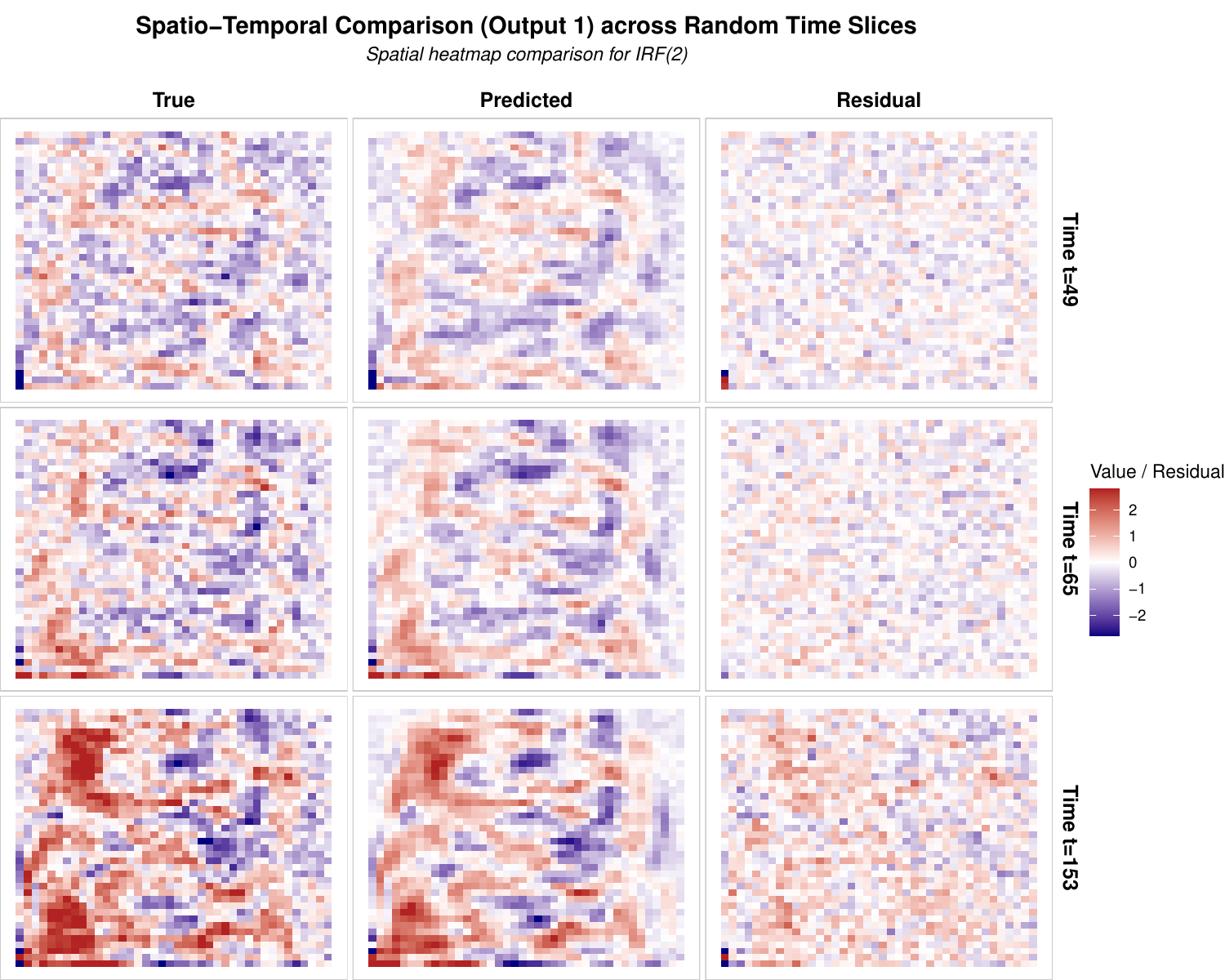}
\caption{Spatial heatmaps at randomly selected three temporal points with $Y_1$ (output 1). Left: True values; Center: Predicted Values; Right: Residuals for IRF(2)}
\label{fig:heatmap_kappa2_y1}
\end{figure}

\FloatBarrier
\subsection{Sensitivity Analysis: Effect of Noise Level on Relative Model Performance}
\label{subsec:noise_sensitivity}

This section examines how model behavior depends on the noise level, varying the noise scale $\eta \in \{1, 5, 10, 20, 30\}$ in the data-generating process under the most demanding $\text{IRF}(2)$ regime. We ask two questions: (i)~is the \emph{accuracy ordering} among the models stable across noise levels, and (ii)~how does the \emph{calibration} of the proposed model's raw predictive intervals depend on the noise level? The headline findings are that the accuracy ordering is stable across the range, while the raw calibration is strongly noise-dependent---broadly improving as the noise rises---which identifies the low-noise under-coverage as an artifact of the high signal-to-noise regime rather than a defect of the method.

\subsubsection*{Theoretical Analysis}

Three mechanisms already established above motivate why the proposed model should remain \emph{robust} across noise levels---retaining its accuracy ordering rather than degrading as the baselines do---without implying a strictly monotone growth of its margin. As $\eta$ rises, $W_\sigma$ learns larger $\hat{\sigma}^2$ in noise-dominated regions, so the precision weighting of Equation~\eqref{eq:precision_weight} suppresses exactly those gradients more strongly (Section~\ref{subsubsec:stage2_mechanics}); the direct model carries the $\Sigma_R^{-1}$ gradient-variance handicap of Theorem~\ref{thm:conditioning} throughout the range, a persistent structural disadvantage rather than one that grows with $\eta$; and the deterministic model, left to converge, interpolates a noise component that scales with $\eta$, which the early-stopping Tikhonov regularization of Stage~1 avoids (Section~\ref{subsubsec:spectral_interp}). The first and third mechanisms therefore protect $M_{\text{prop}}$ as $\eta$ grows, while the second holds the direct baseline back uniformly.

\subsubsection*{Empirical Results}

Table~\ref{tab:noise_sensitivity} reports, for each $\eta$, the overall RMSE of the four models (mean $\pm$ standard deviation over 50 Monte Carlo runs) under $\text{IRF}(2)$, and Table~\ref{tab:noise_calibration} the proposed model's raw and recalibrated PICP and its recalibrated NLL under both scales.

\begin{table}[!ht]
\centering
\caption{Noise sensitivity under IRF(2): overall RMSE (mean $\pm$ SD over 50 Monte Carlo runs per $\eta$). Kriging RMSE is summarized by its median, with the mean in parentheses, since a few hard random-walk realizations right-skew its mean.}
\label{tab:noise_sensitivity}
\begin{tabular}{c c c c c}
\toprule
\textbf{Noise} & \textbf{Proposed} & \textbf{Direct Prob} & \textbf{Determ.} & \textbf{Kriging} \\
\textbf{Scale $\eta$} & \textbf{(Two-stage)} & \textbf{(NLL-only)} & \textbf{(MSE)} & \textbf{med.\ (mean)} \\
\midrule
$1$  & $0.7976 \pm 0.2757$ & $304.08 \pm 44.05$  & $0.9081 \pm 0.3583$ & $0.719\ (1.22)$ \\
$5$  & $0.7857 \pm 0.2712$ & $296.43 \pm 41.03$  & $0.9309 \pm 0.4016$ & $0.740\ (1.19)$ \\
$10$ & $0.8147 \pm 0.2602$ & $284.66 \pm 25.61$  & $0.9817 \pm 0.3871$ & $0.807\ (1.21)$ \\
$20$ & $0.8785 \pm 0.2351$ & $261.16 \pm 28.86$  & $1.0132 \pm 0.3030$ & $0.911\ (1.22)$ \\
$30$ & $0.9087 \pm 0.1714$ & $276.11 \pm 102.51$ & $0.9996 \pm 0.2212$ & $0.952\ (1.21)$ \\
\bottomrule
\end{tabular}
\end{table}

\begin{table}[!ht]
\centering
\caption{Noise sensitivity under IRF(2): calibration of the proposed model across noise levels (50 Monte Carlo runs per $\eta$). Raw PICP is the coverage of the network's uncorrected predictive intervals; the recalibrated PICP applies the coverage-matched scale $s_{\mathrm{cov}}$ of Section~\ref{subsec:exp_setup}, and the recalibrated NLL is reported under both $s_{\mathrm{cov}}$ and the MLE-temperature scale $s_{\mathrm{var}}$. Nominal coverage is $0.95$.}
\label{tab:noise_calibration}
\begin{tabular}{c c c c c}
\toprule
\textbf{Noise} & \textbf{Raw} & \textbf{Recal.\ PICP} & \textbf{Recal.\ NLL} & \textbf{Recal.\ NLL} \\
\textbf{Scale $\eta$} & \textbf{PICP} & ($s_{\mathrm{cov}}$) & ($s_{\mathrm{cov}}$) & ($s_{\mathrm{var}}$) \\
\midrule
$1$  & $0.269$ & $0.820$ & $2.90$   & $1.62$ \\
$5$  & $0.790$ & $0.850$ & $62.05$  & $3.36$ \\
$10$ & $0.896$ & $0.883$ & $144.44$ & $3.59$ \\
$20$ & $0.908$ & $0.919$ & $150.76$ & $3.89$ \\
$30$ & $0.821$ & $0.933$ & $42.01$  & $3.54$ \\
\bottomrule
\end{tabular}
\end{table}

Two findings emerge. First, the accuracy ordering is stable across the entire noise range. The proposed model is the most accurate learned model at every level (lowest mean RMSE among the neural models), and the \emph{direct} probabilistic baseline is worst by a decisive and persistent margin: it \emph{diverges} at every noise level (mean RMSE $261$--$304$ across $\eta$) because its gradient-variance explosion (Theorem~\ref{thm:conditioning}) destabilizes training, and the proposed model attains a lower RMSE on $100\%$ of paired realizations everywhere. The \emph{deterministic} and \emph{kriging} baselines remain comparable to the proposed model throughout---the proposed model is more robust on the deterministic baseline's failure tail and comparable to kriging on the median across the noise range, with kriging's tabulated median falling below the proposed model's mean at the lower noise levels ($\eta=1,5$).

Second, the calibration of the proposed model's \emph{raw} intervals is strongly noise-dependent and broadly improves with $\eta$: the raw PICP rises from $0.27$ at $\eta=1$ to about $0.90$--$0.91$ by $\eta=10$--$20$ (Table~\ref{tab:noise_calibration}; it eases to $0.82$ at $\eta=30$). This identifies the low-noise under-coverage as an artifact of the high signal-to-noise regime---when nearly all variance is explainable signal, the variance head is driven toward overconfidence---rather than a defect of the method. A single coverage recalibration lifts PICP to $0.82$--$0.93$. The recalibrated NLL, however, must be read on the scale that optimizes it: under the MLE-temperature scale $s_{\mathrm{var}}$ the proposed model's NLL is stable across the range ($1.62$ at $\eta=1$ up to $3.89$ at $\eta=20$), whereas the coverage scale $s_{\mathrm{cov}}$---by construction near one in the higher-SNR regimes---leaves the mean NLL tail-inflated at intermediate $\eta$ (up to $151$), the same localized-overconfidence effect discussed in Section~\ref{subsec:sim_result}. On either scale the recalibrated proposed model is far better than the direct baseline; the value of the two-stage design therefore lies in point-forecasting accuracy and robustness, with calibrated coverage after recalibration, rather than in a uniformly small mean NLL.

Taken together, this sweep also locates the primary operating point: the moderate-noise regime around $\eta=10$---used as the main simulation setting (Section~\ref{subsec:sim_result}) and the regime most representative of the real-data application---is where raw coverage is already substantial and recalibration yields near-nominal intervals.

\FloatBarrier
\section{Real Data Analysis: Multi-City Climate Prediction}\label{sec:realdata}

Having characterized the framework on controlled spherical-IRF simulations (Section~\ref{sec:simulation_study}), we now turn to a real-world setting in which the underlying signal and noise structure are unknown. We apply TSCoNet to a multi-city climate dataset covering $50$ major cities worldwide (Figure~\ref{fig:50cities}). For each city, daily precipitation ($\mathtt{PRECTOTCORR}$), minimum temperature ($\mathtt{T2M\_MIN}$), and maximum temperature ($\mathtt{T2M\_MAX}$) were obtained from the NASA POWER database\footnote{\url{https://power.larc.nasa.gov/}}---a global meteorology and solar-energy dataset derived from NASA satellite observations and reanalysis---via the \texttt{nasapower} R package \citep{sparks2018nasapower}, for the period 2000--2020 and aggregated to a monthly resolution (precipitation summed and temperatures averaged within each month), yielding $252$ monthly time points. Each city was treated as a separate spatial location, producing a spatio-temporal array of dimensions $T \times N \times n_{\text{out}}$,\footnote{For the convolutional input, the $N$ cities are arranged as a degenerate $N \times 1$ spatial grid ($n_{\text{lat}}=N$, $n_{\text{lon}}=1$), consistent with the $n_{\text{lat}} \times n_{\text{lon}} \times n_{\text{out}}$ tensor convention of Section~\ref{subsec:architecture}.} with $T=252$ monthly time points, $N=50$ cities and $n_{\text{out}}=3$ variables, denoted $\mathtt{ppt}$, $\mathtt{tmin}$ and $\mathtt{tmax}$ below. Prior to modeling, each variable was standardized to zero mean and unit variance using statistics computed on the training period only (the earliest 70\% of the monthly time points) and applied to the validation and test periods, to avoid information leakage from the held-out data.

\begin{figure}[!ht]
\centering
\includegraphics[width=0.7\linewidth]{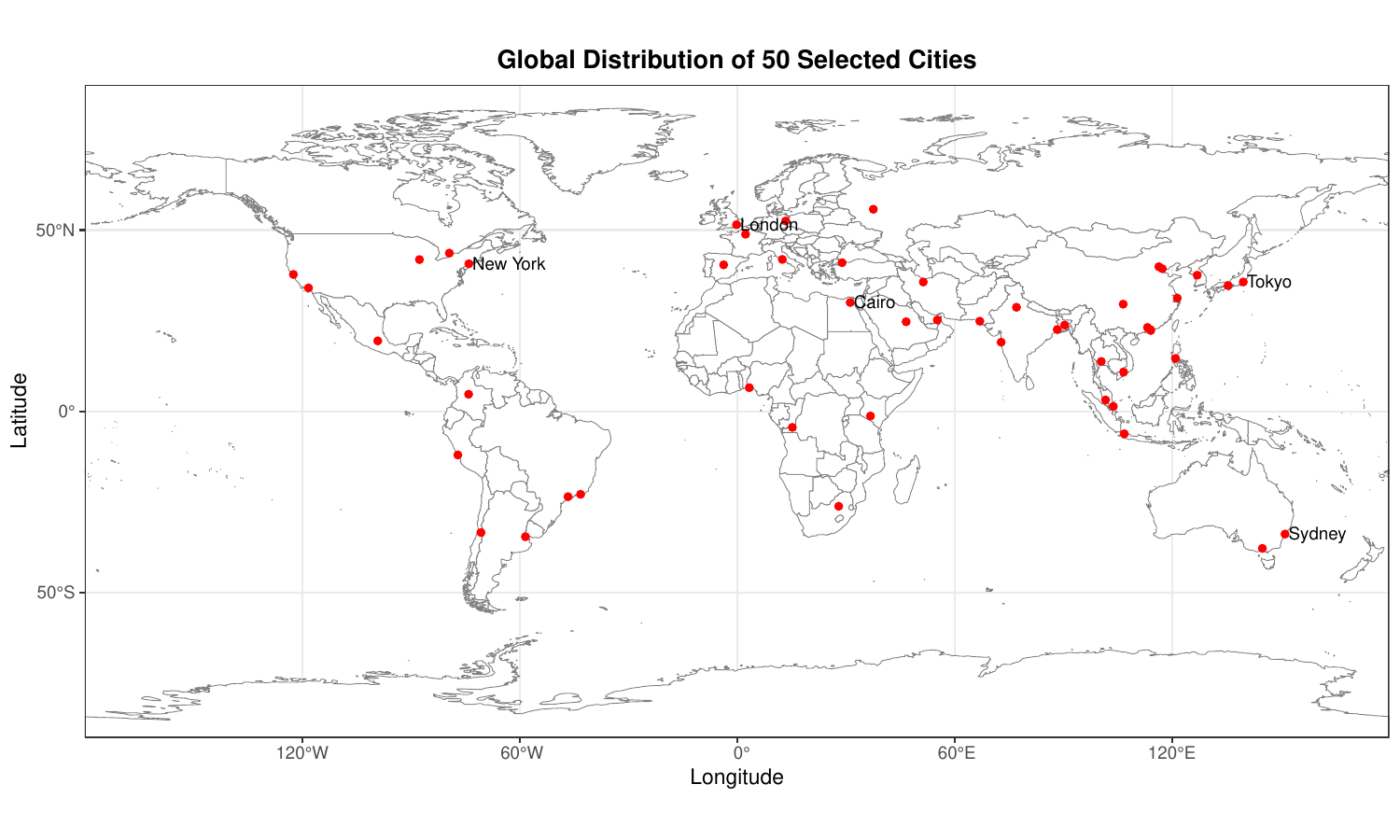}
\caption{Global distribution of the 50 selected cities used in the real-data analysis.}
\label{fig:50cities}
\end{figure}

\subsection{Cross-Variable Dependence in the Climate Data}
\label{subsec:real_dependence}

The three target variables are cross-dependent, and their dependence is \emph{heterogeneous}. Table~\ref{tab:real_corr} reports the empirical correlation among $\mathtt{ppt}$, $\mathtt{tmin}$, and $\mathtt{tmax}$, averaged within city over the observation period so that the values reflect temporal co-variation rather than between-city climate gradients. Minimum and maximum temperature are strongly correlated ($0.80$), whereas precipitation is only weakly related to temperature ($0.25$ with $\mathtt{tmin}$, and essentially uncorrelated, $-0.02$, with $\mathtt{tmax}$). This pattern carries a direct methodological implication: an independent-marginals model would ignore the strong $\mathtt{tmin}$--$\mathtt{tmax}$ coupling, whereas a homogeneous-dependence assumption would misrepresent the near-independence of precipitation. The Gaussian copula layer is precisely what allows the framework to represent this variable-specific dependence, and it is the real-data instance of the cross-variable correlation whose effect on the training gradient we analyze in Section~\ref{subsec:opt_dynamics}. The fitted model recovers this heterogeneous structure: the empirical correlation of its standardized residuals is strong for $\mathtt{tmin}$--$\mathtt{tmax}$ ($\approx 0.71$) and near zero for the precipitation--temperature pairs ($-0.01$ and $-0.07$), matching the pattern in Table~\ref{tab:real_corr}. We read the dependence from these residual correlations rather than from the copula's point estimate of $\mathbf{R}$, which is well identified for the strong pair but high-variance for the near-independent precipitation pairs, where the likelihood is flat (Section~\ref{sec:limitations}).

\begin{table}[!ht]
\centering
\caption{Empirical cross-variable correlation of the three climate variables ($\mathtt{ppt}$, $\mathtt{tmin}$, $\mathtt{tmax}$), averaged within city over 2000--2020. Values are averaged within city so that they reflect temporal co-variation rather than between-city climate gradients.}
\label{tab:real_corr}
\begin{tabular}{l c c c}
\toprule
 & $\mathtt{ppt}$ & $\mathtt{tmin}$ & $\mathtt{tmax}$ \\
\midrule
$\mathtt{ppt}$  & $1.00$  & $0.25$ & $-0.02$ \\
$\mathtt{tmin}$ & $0.25$  & $1.00$ & $0.80$  \\
$\mathtt{tmax}$ & $-0.02$ & $0.80$ & $1.00$  \\
\bottomrule
\end{tabular}
\end{table}

\subsection{Model Setup for Multi-City Prediction}

The $\text{CNN-LSTM}$ model takes a sequence of $L_{\text{in}} = 12$ previous monthly time steps as input and predicts the mean and variance for each climate variable at the next time step. For city $i$ and output variable $k$, the model produces a probabilistic forecast $(\hat{\mu}_{i,k}, \hat{\sigma}_{i,k})$, where $\hat{\mu}_{i,k}$ is the predicted mean and $\hat{\sigma}_{i,k}$ is the predicted standard deviation. 

The model follows the two-stage strategy of Section~\ref{subsubsec:training}, trained with Adam throughout: Stage~1 at learning rate $10^{-3}$ for up to 100 epochs, Stage~1.5 for up to 30 epochs, and Stage~2 at a reduced rate of $10^{-5}$ for up to 50 epochs to prevent overshooting the refined feature space. All stages monitor validation MSE for early stopping (patience~$=10$, restore best weights). Following the same protocol as the simulation study (Section~\ref{subsec:exp_setup}), the monthly series is split chronologically into train/validation/test ($70/15/15$ of the time points), normalization statistics are computed on the training period only, and all reported metrics are computed on the held-out test segment.

\subsection{Prediction Performance}
\label{subsec:real_performance}

Table~\ref{tab:real_metrics} reports the held-out test RMSE (mean $\pm$ standard deviation over 50 model-init seeds) for each climate variable, and Table~\ref{tab:real_uq} the recalibrated predictive uncertainty (95\% prediction-interval coverage, PICP, and average per-output NLL). TSCoNet and the deterministic baseline are effectively tied on accuracy (the deterministic baseline marginally lower, by less than the seed-to-seed spread), and both substantially outperform the direct-NLL and kriging baselines; after recalibration the proposed model's intervals track the nominal $0.95$ coverage.

\begin{table}[!ht]
\centering
\caption{Held-out test RMSE (mean $\pm$ standard deviation over 50 model-init seeds) on the real climate dataset. Bold marks the lowest mean RMSE per variable; the kriging baseline is deterministic in the seed, hence SD $=0$.}
\label{tab:real_metrics}
\begin{tabular}{l c c c}
\toprule
\textbf{Model} & $\mathtt{ppt}$ & $\mathtt{tmin}$ & $\mathtt{tmax}$ \\
\midrule
\textbf{Proposed (Two-stage)}     & $0.7349 \pm 0.0019$ & $0.1372 \pm 0.0036$ & $0.1688 \pm 0.0036$ \\
Direct (NLL-only)                 & $694.93 \pm 18.04$ & $698.18 \pm 18.15$ & $697.32 \pm 18.14$ \\
Deterministic Baseline            & $\mathbf{0.7341 \pm 0.0016}$ & $\mathbf{0.1330 \pm 0.0032}$ & $\mathbf{0.1656 \pm 0.0032}$ \\
Kriging Baseline                  & $0.9412 \pm 0.0000$ & $0.3289 \pm 0.0000$ & $0.3481 \pm 0.0000$ \\
\bottomrule
\end{tabular}
\end{table}

\begin{table}[!ht]
\centering
\caption{Recalibrated predictive uncertainty on the held-out test set (mean $\pm$ standard deviation over 50 model-initialization seeds), for every model in Table~\ref{tab:real_metrics}. PICP and NLL are given under both recalibration scales as $s_{\mathrm{cov}}/s_{\mathrm{var}}$, PICP targeting the nominal $0.95$. Bold marks the lowest mean RMSE and the best NLL under each scale; a dash means the model supplies no predictive distribution.}
\label{tab:real_uq}
\begin{tabular}{l c c c}
\toprule
\textbf{Model} & \textbf{Overall RMSE} & \textbf{PICP} ($s_{\mathrm{cov}}/s_{\mathrm{var}}$) & \textbf{NLL} ($s_{\mathrm{cov}}/s_{\mathrm{var}}$) \\
\midrule
\textbf{Proposed (Two-stage)}     & $0.4425 \pm 0.0013$ & $0.946/0.949$ & $\mathbf{0.032/-0.008}$ \\
Direct (NLL-only)                 & $696.81 \pm 18.11$  & $0.951/1.000$ & $8.653/7.978$ \\
Deterministic Baseline            & $\mathbf{0.4412 \pm 0.0012}$ & --            & --            \\
Kriging Baseline                  & $0.6097 \pm 0.0000$ & $0.944/0.938$ & $0.684/0.681$ \\
\addlinespace[3pt]
Stop-gradient                     & $1.1148 \pm 0.0676$ & $0.948/0.947$ & $1.467/1.465$ \\
$\beta$-NLL (scratch)             & $691.38 \pm 22.47$  & $0.950/1.000$ & $8.927/8.368$ \\
$\beta$-NLL (MSE warm-up)         & $0.7339 \pm 0.3986$ & $0.943/0.943$ & $1.103/0.962$ \\
\bottomrule
\end{tabular}
\end{table}

Consistent with the simulation study, the real-data results show the same core ordering. On point accuracy (Table~\ref{tab:real_metrics}) the deterministic baseline is slightly but consistently ahead: it attains the lower RMSE on $74\%$ of the 50 model-initialization seeds, and a paired $t$-test detects the difference ($p\approx2\times10^{-5}$). The gap is nevertheless too small to be meaningful. It is $0.3\%$ of overall RMSE ($0.4425$ against $0.4412$) and at most $3.2\%$ on any single variable ($\mathtt{tmin}$); it is smaller than the proposed model's own seed-to-seed standard deviation ($0.0013$), and the two error distributions overlap---the proposed model's best seed ($0.4394$) is more accurate than the deterministic baseline's worst ($0.4444$). The paired test establishes the sign of the difference, not that it is consequential. We accordingly describe the proposed model as reaching deterministic-level accuracy, at a cost that buys prediction intervals the deterministic model cannot supply at all. Both clearly outperform the direct-NLL baseline, which \emph{diverges} on the real data exactly as in simulation ($\approx 695$ RMSE per variable; the proposed model is lower on all $50$ seeds), and the scratch $\beta$-NLL baseline diverges likewise ($\approx 691$); adding only the Stage-1 MSE warm-up to $\beta$-NLL restores competitive accuracy ($0.734$ overall), and the variance-shielding stop-gradient baseline trains stably but is less accurate ($1.11$ overall). After recalibration (Table~\ref{tab:real_uq}) the proposed model's $95\%$ coverage is close to nominal ($\approx 0.946$ under $s_{\mathrm{cov}}$) and---unlike the simulated IRF(2) regime---its two recalibration scales agree, giving a well-behaved marginal NLL under both scales ($\approx 0.0$) that far surpasses the direct baseline ($\approx 8.7$).

The comparison with the stop-gradient baseline reverses here. In simulation that baseline held the coverage advantage (Section~\ref{subsec:sim_result}); on the real data every model that trains stably reaches near-nominal coverage---$0.946$ for the proposed model against $0.948$ for stop-gradient and $0.943$ for the warm-up $\beta$-NLL---so coverage no longer separates them, and the separation moves to likelihood, where the proposed model attains $\text{NLL}\approx0.03$ against $1.47$ and $1.10$. It does so while also being the more accurate forecaster ($0.443$ overall RMSE against $1.115$ and $0.734$), and by a margin that leaves no ambiguity: it is lower on all $50$ seeds against each of the three, and its worst seed ($0.4456$) is still better than the best seed of the closest of them ($0.4572$ for the warm-up $\beta$-NLL). The reason the simulated coverage gap does not recur is visible in the diagnostics: the proposed model's variance head leaves only $9.4\%$ of location--variable pairs at the floor here, against about half at IRF(0)/IRF(1), confirming that the collapse behind its simulated under-coverage is a feature of those high signal-to-noise synthetic fields rather than of the schedule itself.

The proposed model also outperforms the space-time kriging forecast baseline on every variable and every seed (e.g.\ $\mathtt{ppt}$ $0.7349$ vs.\ $0.9412$). Importantly, this kriging baseline is itself a genuine forecaster---it predicts each held-out month from a $12$-month past window using the cities' true coordinates and a space-time variogram---so the gap is not because kriging ignores time; rather, it reflects the gain from \emph{learned, nonlinear} seasonal-temporal and cross-variable structure over a linear, stationarity-assuming spatio-temporal interpolator. The same holds for its predictive intervals (Table~\ref{tab:real_uq}): the two reach comparable recalibrated coverage ($0.944$ for kriging against $0.946$), but kriging's marginal NLL is an order of magnitude worse ($0.684$ against $0.032$), and its raw intervals over-cover ($0.982$) and require a substantial correction ($s_{\mathrm{cov}}=0.69$)---in contrast to the simulated fields, where they needed essentially none (Section~\ref{subsec:sim_result}). This is the expected consequence of leaving the regime in which the kriging model is close to correctly specified. We also note that the robustness advantage observed in simulation (the proposed model's thinner failure tail on hard realizations) does \emph{not} apply here: the real dataset is fixed, so the only source of variation is model initialization, the seed-to-seed RMSE spread is small, and the deterministic baseline shows no failure tail---robustness is therefore a simulation-only finding.

\FloatBarrier
\subsection{Where Predictive Uncertainty Concentrates}
\label{subsec:real_uncertainty}

The cross-variable dependence structure of Section~\ref{subsec:real_dependence} also shapes \emph{where} predictive uncertainty concentrates. Precipitation is intrinsically far less predictable than temperature---its held-out RMSE ($0.73$) is $4.4$--$5.4$ times that of $\mathtt{tmin}$ and $\mathtt{tmax}$ ($0.14$--$0.17$; Table~\ref{tab:real_metrics})---and because the recalibrated intervals attain near-nominal $95\%$ coverage across all three variables (Table~\ref{tab:real_uq}), the framework necessarily assigns correspondingly wider predictive intervals to precipitation. The variance ``noise sink'' thus absorbs the large aleatoric variability of precipitation rather than allowing it to corrupt the mean forecast. This is also where the contrast with the kriging baseline's intervals is sharpest. The proposed model's predictive standard deviation varies more strongly here than in any simulated regime---coefficient of variation $1.321$ across locations and $0.469$ across time, against $0.030$ and exactly $0$ for kriging (Appendix~H). The practical contribution on the real data is accordingly not a point-accuracy gain over the deterministic baseline---the two are tied---but a \emph{calibrated, variable-specific, and spatially resolved} characterization of predictive uncertainty across the 50 cities, information the deterministic model cannot provide.

\subsection{Visualization of Predictions}

City-wise mean predictions for precipitation, the least predictable of the three variables, are shown in Figure~\ref{fig:ppt_true_pred}. The figures in this subsection are illustrative and span the full series; all quantitative metrics reported above are computed on the held-out test period only.

\begin{figure}[!ht]
\centering
\includegraphics[width=0.8\linewidth]{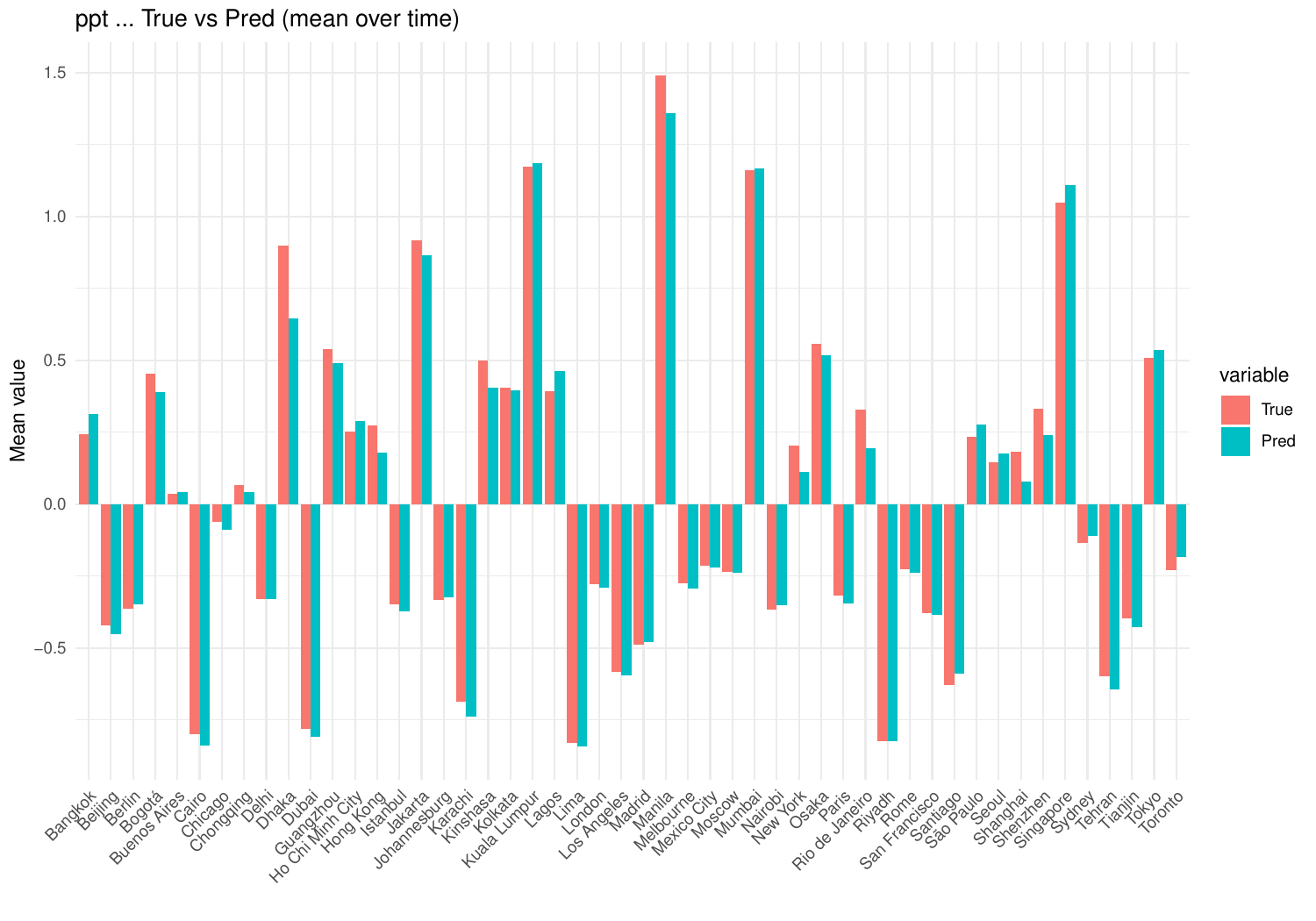}
\caption{City-wise comparison of true and predicted means (over time) for Precipitation ($\mathtt{ppt}$).}
\label{fig:ppt_true_pred}
\end{figure}

Per-city time series across all three variables, the city-wise residual means, and the corresponding true-versus-predicted plots for $\mathtt{tmin}$ and $\mathtt{tmax}$ are provided in Appendix~E. Together they indicate that the model maintains consistent accuracy across locations and captures the seasonal dynamics of all three variables, with residuals centered near zero and the largest deviations---as the RMSE ordering of Section~\ref{subsec:real_uncertainty} already implies---in precipitation. These visual comparisons corroborate the quantitative results of Sections~\ref{subsec:real_performance}--\ref{subsec:real_uncertainty} at the level of individual cities.

\section{Limitations and Future Work}\label{sec:limitations}

While the proposed framework performs well, several extensions remain. The 50 real-world stations form an irregular, arbitrarily ordered array rather than a regular grid, so the spatial convolution's advantage---validated in the gridded simulation study---is less fully exercised on this irregular network; here the temporal and cross-variable components play a larger role. This does not affect the framework's central contribution---calibrated probabilistic forecasting---and a distance- or graph-based encoder (e.g.\ a graph-convolutional network) is a natural extension for irregular station networks.

A further direction is non-Gaussian marginals: precipitation is skewed and zero-inflated, so zero-inflated Gamma or Tweedie distributions integrated into the dual-head output would yield more physically consistent intervals than the Gaussian marginals used here. The Gaussian marginals are load-bearing for two of our results: the equivalence of Appendix~B and the gradient-conditioning argument of Section~\ref{subsec:opt_dynamics}, whose proof uses a score linear in the standardized residual. Both would have to be re-derived, and the atom at zero would additionally require the copula likelihood to admit a discrete component. A final direction is spherical CNN architectures, which would align the convolution operator with the underlying spherical geometry for global-scale forecasting, a natural direction given the spherical formulation of this work.

Finally, the copula's point estimate of $\mathbf{R}$ is imprecise at both ends of the dependence range. Where the true cross-variable dependence is near zero the joint likelihood is nearly flat in the corresponding off-diagonal, so the estimate is high-variance across model-initialization seeds even though the empirical correlation of the standardized residuals is stable; we therefore read the learned dependence from the residual correlation, which is the more stable diagnostic. Where the dependence is extreme the estimate instead saturates against the boundary of the correlation space ($|\hat{\mathbf{R}}_{ij}|>0.999$ in $31$ of $50$ runs at IRF(2)), and the fixed shrinkage $\alpha=0.05$ is what keeps the operative matrix away from singularity---so $\alpha$ is a precondition for the joint predictive law in that regime, not only a stabilizer of the training gradient, and it is set a priori rather than adapted. A parameterization that keeps the estimate strictly inside the positive-definite cone with a bounded condition number would remove both failure modes at once.

\section{Conclusion}

In this study, we developed TSCoNet, a two-stage, copula-based probabilistic CNN-LSTM framework for multivariate spatio-temporal forecasting. A key limitation of standard deterministic models is that they provide no measure of predictive uncertainty, whereas directly optimizing the negative log-likelihood supplies uncertainty but degrades point accuracy. TSCoNet resolves this tension with a two-stage scheme: a first stage anchors the mean under an MSE objective to deterministic-level accuracy, after which the mean head is frozen and, during a second NLL stage, the backbone and a dedicated variance head are refined. Freezing the mean head prevents it from absorbing high-frequency noise and routes residual error into the variance head, which acts as a dedicated noise sink. It uses precision-weighted gradients to drive \textit{Uncertainty-Aware Feature Refinement} within the shared backbone, preserving this accuracy while producing variance estimates that, after recalibration, yield calibrated predictive intervals.

Simulation results under Intrinsic Random Functions (IRF) and real-world multi-city climate experiments support the value of this two-stage mechanism. TSCoNet supplies dynamic, cross-variable uncertainty bounds that---after a standard post-hoc recalibration---are well-calibrated, while matching deterministic-level point accuracy. In simulation the direct-NLL baseline diverges under its gradient-variance explosion---an ablation confirms this reflects the training schedule, since the same objective recovers once a mean warm-up is added---whereas TSCoNet matches the deterministic and kriging baselines (beating them at low non-stationarity); on real data it matches the deterministic baseline's accuracy (the deterministic baseline marginally lower), outperforms the direct-NLL and kriging baselines, and supplies calibrated uncertainty that the deterministic baseline lacks. Ultimately, this uncertainty-aware framework provides a theoretically grounded and reliable tool for multivariate spatio-temporal forecasting, with applications ranging from climate and environmental modeling to broader scientific and industrial domains.

\section*{Declaration of Generative AI Use}

Claude (Anthropic) was used to assist with code development and with language
editing of the manuscript. No generative AI tool was used to generate, analyze,
or interpret the data, and no AI tool is listed as an author. The authors
designed the methodology, conducted all analyses, verified all numerical
results, and take full responsibility for the content of this article.

\section*{Data Availability Statement}

The real-world climate data analyzed in this article are publicly available from
the NASA POWER database (\url{https://power.larc.nasa.gov/}) and were accessed
through the \texttt{nasapower} \textsf{R} package \citep{sparks2018nasapower}.
The simulated data are generated by the analysis code, which is provided as
supplementary material together with the stored Monte Carlo result files needed
to reproduce every reported table and figure.

\acks{\textbf{Funding.} This work was supported in part by an Aspire Alpha Support
award from Ball State University. The funder had no role in the study design, the
analysis, or the decision to submit.

\textbf{Competing interests.} The authors report there are no competing
interests to declare.}

\phantomsection\label{supplementary-material}
\bigskip

\begin{center}

{\large\bf SUPPLEMENTARY MATERIAL}

\end{center}

\begin{description}

\item[Appendices A--H:]
Included at the end of this document: proofs; the copula/multivariate-Gaussian
equivalence; a neural-tangent-kernel analysis of noise interpolation and early
stopping; additional simulation and real-data figures; the Stage-2 gradient
derivation; the per-variable RMSE and coverage definitions; and the comparison
against the kriging baseline's predictive intervals.

\item[Analysis code and stored results:]
A single archive containing (i)~the \textsf{R} scripts implementing the
data-generating process, the proposed two-stage model and all baselines, the
spatio-temporal kriging forecaster, and the Monte Carlo drivers used to produce
the reported tables, and (ii)~the stored $50$-run simulation and $50$-seed
real-data result files, together with two read-only summary scripts that
regenerate every reported table from them without re-running any model fit.
(zip archive)

\end{description}

\bigskip

\bibliography{references}

\clearpage
\setcounter{table}{0}
\setcounter{figure}{0}
\renewcommand{\thetable}{S\arabic{table}}
\renewcommand{\thefigure}{S\arabic{figure}}
\appendix

\section{Proofs of the Main Results}

\begin{proof}[Proof of Theorem~1]
Recall the data-generating process $Y = f^*(X) + \epsilon$ with $\epsilon \sim \mathcal{N}(\bm{0}, \Sigma_R)$, so that the conditional mean of the response is the true signal, $\mathbb{E}_{Y \mid X}[Y \mid X] = f^*(X)$, since the noise has zero mean and is independent of $X$.
For probabilistic models minimizing $\mathcal{L}_{\text{NLL}}$, the gradient with respect to $\mu_\theta$ is (up to a positive constant arising from sample-size normalization):
\begin{equation}
\nabla_\mu \mathcal{L}_{\text{NLL}}(\theta) \propto -\Sigma_R^{-1} (Y - \mu_\theta(X))
\end{equation}
Under the infinite-data assumption, the empirical gradient converges to its population counterpart by the law of large numbers, so it suffices to analyze the population gradient. Taking the expectation over the joint distribution of $(X,Y)$ and applying the Law of Iterated Expectations $\mathbb{E}_{X,Y}[\cdot] = \mathbb{E}_X[\mathbb{E}_{Y \mid X}[\cdot \mid X]]$:
\begin{equation}
\mathbb{E}_{X,Y}[\nabla_\mu \mathcal{L}_{\text{NLL}}(\theta)] = \mathbb{E}_X \left[ -\Sigma_R^{-1} (\mathbb{E}_{Y \mid X}[Y \mid X] - \mu_\theta(X)) \right] = \mathbb{E}_X \left[ -\Sigma_R^{-1} (f^*(X) - \mu_\theta(X)) \right]
\end{equation}
Setting the expected gradient to zero and noting that $\Sigma_R$ is a positive-definite noise covariance, hence invertible, so that $\Sigma_R^{-1}$ is a fixed matrix that factors out of $\mathbb{E}_X[\cdot]$ and can be cancelled:
\begin{equation}
\mathbb{E}_X\!\left[ f^*(X) - \mu_\theta^*(X) \right] = \bm{0}.
\end{equation}
Because $\mu_\theta$ is optimized over a sufficiently expressive function class---so that its value can be chosen independently at each $X$---minimizing the population risk reduces to a pointwise minimization at every $X$. The stationarity condition therefore holds not merely on average but for each $X$ separately, giving
\begin{equation}
f^*(X) - \mu_\theta^*(X) = \bm{0} \implies \mu_\theta^*(X) = f^*(X) \quad \text{for all } X.
\end{equation}
This confirms that the theoretical global minimum $\mu_\theta^*$ is the true conditional mean $f^*(X)$, completely independent of $\mathbf{R}$. Consequently, the theoretical RMSE remains invariant to changes in $\mathbf{R}$.
\end{proof}

\begin{proof}[Proof of Theorem~2]
During training, models are updated using SGD. The stability is governed by the covariance of the stochastic gradient $g = \nabla_\mu \mathcal{L}_{\text{NLL}}$. Evaluating $g$ at the true mean $\mu = f^*(X)$:
\begin{equation}
\Cov(g \mid X) = \Cov(-\Sigma_R^{-1}(Y - \mu) \mid X) = \Sigma_R^{-1} \Cov(Y \mid X) \Sigma_R^{-1} = \Sigma_R^{-1} \Sigma_R \Sigma_R^{-1} = \Sigma_R^{-1}
\end{equation}
The gradient variance in $M_{\text{direct}}$ is precisely the precision matrix $\Sigma_R^{-1}$. As the correlation $\mathbf{R}$ strengthens, $\Sigma_R$ becomes ill-conditioned, and its smallest eigenvalue approaches 0, causing the largest eigenvalue of $\Sigma_R^{-1}$ to explode towards infinity ($\lambda_{\max}(\Sigma_R^{-1}) \to \infty$). Although the true asymptotic optimum is invariant to $\mathbf{R}$ (Theorem~1), in the finite-sample SGD regime the optimizer cannot traverse this ill-conditioned loss landscape reliably within a bounded number of steps. This artificially inflates the empirical RMSE.

By contrast, $M_{\text{prop}}$ dedicates Stage 1 exclusively to minimizing the Mean Squared Error ($\mathcal{L}_{\text{MSE}}$), completely decoupling $\mu$ from $\Sigma_R$. The gradient is $g_{\text{prop}} = -(Y-\mu)$.
\begin{equation}
\Cov(g_{\text{prop}} \mid X) = \Cov(-(Y-\mu) \mid X) = \Sigma_R
\end{equation}
\end{proof}

\begin{proof}[Proof of Proposition~3]
We relax the simplifying assumption $\mathbf{D}=\mathbf{I}$ and consider the general case $\mathbf{D} \neq \mathbf{I}$, so that $\Sigma_R = \mathbf{D}\mathbf{R}\mathbf{D}$ with arbitrary marginal standard deviations. Since $\mathbf{R}$ has unit diagonal, the diagonal entries of $\Sigma_R$ are $(\Sigma_R)_{ii} = \sigma_i^2$, so its trace is the sum of the marginal variances:
\begin{equation}
\Tr(\Sigma_R) = \sum_{i=1}^K (\Sigma_R)_{ii} = \sum_{i=1}^K \sigma_i^2,
\end{equation}
which is independent of $\mathbf{R}$ by the marginal property noted above. Because $\Sigma_R$ is symmetric, the Spectral Theorem permits the eigendecomposition $\Sigma_R = \mathbf{Q} \mathbf{\Lambda} \mathbf{Q}^\top$. Applying the cyclic property of the trace ($\Tr(\mathbf{A}\mathbf{B}\mathbf{C}) = \Tr(\mathbf{C}\mathbf{A}\mathbf{B})$):
\begin{equation}
\Tr(\Sigma_R) = \Tr(\mathbf{Q} \mathbf{\Lambda} \mathbf{Q}^\top) = \Tr(\mathbf{Q}^\top \mathbf{Q} \mathbf{\Lambda}) = \Tr(\mathbf{\Lambda}) = \sum_{i=1}^K \lambda_i.
\end{equation}
This identity $\sum \sigma_i^2 = \sum \lambda_i$ shows that although strong correlation $\mathbf{R}$ may redistribute the individual eigenvalues $\lambda_i$, the total stochastic energy (the trace) is invariant to $\mathbf{R}$ and bounded by the sum of marginal variances. Since $\Sigma_R$ is positive semi-definite, all eigenvalues are non-negative, so its largest eigenvalue is bounded by the trace, $\lambda_{\max}(\Sigma_R) \leq \Tr(\Sigma_R) = \sum_{i=1}^K \sigma_i^2$. Hence the spectral norm of the Stage-1 gradient covariance is controlled by the ($\mathbf{R}$-independent) marginal variances and cannot blow up as $\mathbf{R}$ becomes ill-conditioned. This is in stark contrast to the direct-NLL gradient covariance $\Sigma_R^{-1}$, whose largest eigenvalue $\lambda_{\max}(\Sigma_R^{-1}) = 1/\lambda_{\min}(\Sigma_R)$ diverges as $\mathbf{R}$ drives $\lambda_{\min}(\Sigma_R) \to 0$.
\end{proof}

\section{Equivalence of the Gaussian Copula PDF and the Multivariate Negative Log-Likelihood}\label{app:equivalence}
In this appendix, we establish the equivalence between our copula-factorized training loss and the multivariate Gaussian negative log-likelihood. While the underlying relationship is classical \citep{sklar1959fonctions, nelsen2007introduction}, we derive it explicitly for our decomposed, spatially-averaged formulation, making the connection between the implemented copula loss of the main text and standard maximum-likelihood estimation precise. Concretely, we show that the three formulations used in the proposed model---the joint probability density function based on the Gaussian copula, the empirical spatial-averaged loss used for the deep-learning implementation, and the theoretical multivariate NLL---are mathematically equivalent from an optimization perspective. Throughout, we assume Gaussian marginals, $f_k = \mathcal{N}(\hat{\mu}_{i,j,k}, \hat{\sigma}_{i,j,k}^2)$, together with a Gaussian copula; the equivalence established below is specific to this Gaussian specification.

At a spatial grid point $(i, j)$, assuming the model predicts a marginal Gaussian distribution $\mathcal{N}(\hat{\mu}_{i,j,k}, \hat{\sigma}_{i,j,k}^2)$, the joint PDF using a Gaussian Copula is defined as:
\begin{equation}
    f(\mathbf{Y}_{i,j}) = \left[ \prod_{k=1}^{n_{\text{out}}} f_k(Y_{i,j,k} \mid \hat{\mu}_{i,j,k}, \hat{\sigma}_{i,j,k}) \right] \cdot \frac{1}{\sqrt{|\mathbf{R}|}} \exp\left( -\frac{1}{2} \mathbf{w}_{i,j}^T (\mathbf{R}^{-1} - \mathbf{I}) \mathbf{w}_{i,j} \right) \label{eq:app_pdf}
\end{equation}
where $\mathbf{w}_{i,j}$ is the vector of standardized residuals (Z-scores) such that $w_{i,j,k} = (Y_{i,j,k} - \hat{\mu}_{i,j,k}) / \hat{\sigma}_{i,j,k}$, and $\mathbf{R}$ is the correlation matrix. The empirical loss function computed over the entire spatial grid ($N_{\text{grid}} = n_{\text{lat}} \times n_{\text{lon}}$) is:
\begin{multline}
    \mathcal{L}_{\text{joint}} = \frac{1}{N_{\text{grid}}} \sum_{i,j} \Bigg[ \sum_{k=1}^{n_{\text{out}}} \left( \log \hat{\sigma}_{i,j,k} + \frac{(Y_{i,j,k} - \hat{\mu}_{i,j,k})^2}{2\hat{\sigma}_{i,j,k}^2} \right) \\
    + \frac{1}{2}\log|\mathbf{R}| + \frac{1}{2}\mathbf{w}_{i,j}^T(\mathbf{R}^{-1} - \mathbf{I})\mathbf{w}_{i,j} \Bigg] \label{eq:app_emp}
\end{multline}
Conversely, the standard multivariate Gaussian NLL, often used in statistics, is compactly expressed using the covariance matrix $\Sigma_R$:
\begin{equation}
    \mathcal{L}_{\text{NLL}}(\theta) = \frac{1}{2} \log |\Sigma_R| + \frac{1}{2}(Y - \mu_\theta)^\top \Sigma_R^{-1} (Y - \mu_\theta) + \frac{K}{2} \log(2\pi) \label{eq:app_nll}
\end{equation}

Deep learning optimization converts Maximum Likelihood Estimation (MLE) into a minimization of the negative logarithm. Taking $-\log$ of the single-point PDF in Equation \eqref{eq:app_pdf} yields:
\begin{equation}
    -\log f(\mathbf{Y}_{i,j}) = -\sum_{k=1}^{n_{\text{out}}} \log f_k(Y_{i,j,k}) - \log \left( \text{Copula}(\mathbf{w}_{i,j}) \right)
\end{equation}
Substituting the Gaussian marginal density
\[ f_k(Y_{i,j,k} \mid \hat{\mu}_{i,j,k}, \hat{\sigma}_{i,j,k}) = \frac{1}{\sqrt{2\pi}\,\hat{\sigma}_{i,j,k}} \exp\!\left( -\frac{(Y_{i,j,k} - \hat{\mu}_{i,j,k})^2}{2\hat{\sigma}_{i,j,k}^2} \right), \]
the first term (marginal NLL) evaluates to:
\begin{equation}
    -\sum_{k=1}^{n_{\text{out}}} \log f_k(Y_{i,j,k}) = \sum_{k=1}^{n_{\text{out}}} \left( \log \hat{\sigma}_{i,j,k} + \frac{(Y_{i,j,k} - \hat{\mu}_{i,j,k})^2}{2\hat{\sigma}_{i,j,k}^2} \right) + \frac{n_{\text{out}}}{2}\log(2\pi) \label{eq:marginal_NLL_part}
\end{equation}
The constant term $\frac{n_{\text{out}}}{2}\log(2\pi)$ vanishes upon differentiation. Applying $-\log$ to the copula term gives:
\begin{equation}
    -\log \left( \frac{1}{\sqrt{|\mathbf{R}|}} \exp\left( -\frac{1}{2} \mathbf{w}_{i,j}^T (\mathbf{R}^{-1} - \mathbf{I}) \mathbf{w}_{i,j} \right) \right) = \frac{1}{2}\log|\mathbf{R}| + \frac{1}{2}\mathbf{w}_{i,j}^T(\mathbf{R}^{-1} - \mathbf{I})\mathbf{w}_{i,j} \label{eq:copula_part}
\end{equation}
Summing these terms from Equations \eqref{eq:marginal_NLL_part} and \eqref{eq:copula_part} while averaging over the spatial grid $N_{\text{grid}}$ yields the empirical loss function $\mathcal{L}_{\text{joint}}$ presented in Equation \eqref{eq:app_emp}.

To prove equivalence with the theoretical multivariate NLL (Equation \eqref{eq:app_nll}), we decompose the covariance matrix at a single grid point as $\Sigma_R = \mathbf{D} \mathbf{R} \mathbf{D}$, where $\mathbf{D} = \text{diag}(\sigma_1, \dots, \sigma_K)$. The determinant property yields:
\begin{equation}
    |\Sigma_R| = |\mathbf{D} \mathbf{R} \mathbf{D}| = |\mathbf{D}|^2 |\mathbf{R}| = \left( \prod_{k=1}^K \sigma_k^2 \right) |\mathbf{R}|
\end{equation}
Taking the logarithm and multiplying by $\frac{1}{2}$:
\begin{equation}
    \frac{1}{2} \log |\Sigma_R| = \sum_{k=1}^K \log \sigma_k + \frac{1}{2} \log |\mathbf{R}|
\end{equation}
This perfectly matches the first two corresponding terms in Equation \eqref{eq:app_emp}. Next, rewriting the quadratic form from Equation \eqref{eq:app_nll} using $\Sigma_R^{-1} = \mathbf{D}^{-1} \mathbf{R}^{-1} \mathbf{D}^{-1}$ and $\mathbf{w} = \mathbf{D}^{-1}(Y - \mu)$:
\begin{equation}
    \frac{1}{2}(Y - \mu)^\top \Sigma_R^{-1} (Y - \mu) = \frac{1}{2} \mathbf{w}^\top \mathbf{R}^{-1} \mathbf{w} \label{eq:app_quad1}
\end{equation}
By factoring the residual term from Equation \eqref{eq:app_emp} as $\frac{1}{2} \mathbf{w}^\top \mathbf{I} \mathbf{w}$ and summing it with the copula cross-term, we obtain perfect cancellation:
\begin{equation}
    \frac{1}{2} \mathbf{w}^\top \mathbf{I} \mathbf{w} + \frac{1}{2}\mathbf{w}^\top(\mathbf{R}^{-1} - \mathbf{I})\mathbf{w} = \frac{1}{2} \mathbf{w}^\top \mathbf{R}^{-1} \mathbf{w} \label{eq:app_quad2}
\end{equation}
Equations \eqref{eq:app_quad1} and \eqref{eq:app_quad2} confirm that, excluding constants, the empirical loss $\mathcal{L}_{\text{joint}}$ is algebraically identical to the theoretical multivariate NLL. This establishes that the decomposed empirical implementation strictly shares the same optimization landscape.

\section{Spectral Interpolation of Noise and Early Stopping: A Neural Tangent Kernel Analysis}

\subsection{Spectral Interpolation of Noise under Prolonged MSE Training}
\textbf{Problem (Spectral Interpolation).} Left to converge, MSE training interpolates not only the underlying signal but also the high-frequency noise, so that without early stopping the deterministic model memorizes noise. To trace the learning dynamics underlying this, we use Neural Tangent Kernel (NTK) theory \citep{jacot2018neural} for a network optimized under the MSE objective, which describes both $M_{\text{det}}$ and Stage~1 of $M_{\text{prop}}$. Let the kernel matrix be defined by the gradient features: $\Theta_t = (\nabla_W \mu_t)(\nabla_W \mu_t)^\top$. Under the infinite-width limit, weights barely move from initialization ($W_t \approx W_0$), and the kernel is constant ($\Theta_t \approx \Theta_0 \equiv \Theta$). 
We additionally assume zero initialization ($\mu_0 = 0$), which together with the constant-kernel approximation yields a tractable closed-form solution. The optimization dynamics of Gradient Descent under loss $L = \frac{1}{2}(\mu_t - Y)^2$ can be formulated using the chain rule:
\begin{equation}
\frac{d \mu_t}{dt} = (\nabla_W \mu_t)^\top \frac{dW}{dt} = (\nabla_W \mu_t)^\top \left[ -\alpha (\nabla_W \mu_t) (\mu_t - Y) \right] = -\alpha \Theta (\mu_t - Y)
\end{equation}
Solving this differential equation with initial condition $\mu_0 = 0$ yields the closed-form trajectory:
\begin{equation}
\mu_t(X_{\text{train}}) = Y_{\text{train}} - \exp(-\alpha \Theta t)Y_{\text{train}}
\end{equation}
As $t \to \infty$, $\mu_t \to Y_{\text{train}} = f^*(X_{\text{train}}) + \epsilon_{\text{train}}$. The model interpolates both the smooth signal $f^*$ and the unpredictable high-frequency noise $\epsilon$, suffering from the ``spectral interpolation of noise'' \citep{belkin2019reconciling}. This limit describes training run to convergence without early stopping; in practice every MSE-trained model in this work---including $M_{\text{det}}$---is halted well before this regime by validation-MSE early stopping, as detailed next.

\subsection{Early Stopping as Tikhonov Regularization}
\textbf{Mechanism (Early Stopping).} Halting Stage~1 at an optimal time $\tau$ acts as an implicit low-pass filter that preserves the signal while attenuating the noise. Interrupting MSE training in this way yields the prediction $\mu_\tau(X) = (\mathbf{I} - \exp(-\alpha \Theta \tau))Y \approx f^*(X)$, which follows the same spectral-shrinkage path as Tikhonov (L2) regularization under the calibration $\lambda \propto 1/\tau$, attaining comparable estimation risk \citep{ali2019continuous}. Severely attenuating the stochastic noise $\epsilon$ in this manner is what saves the model from memorization. We emphasize that this early-stopping mechanism is not specific to $M_{\text{prop}}$: it applies to any MSE-trained model, including the deterministic baseline $M_{\text{det}}$, and Stage~1 of $M_{\text{prop}}$ is identical in this respect. The distinguishing component of $M_{\text{prop}}$ is Stage~2 (analyzed in the main text), where precision-weighted gradients further refine the shared backbone.

\section{Additional Simulation Figures}
These figures complement the IRF(2) representatives shown in the main text.

\begin{figure}[H]
\centering
\includegraphics[width=0.8\linewidth]{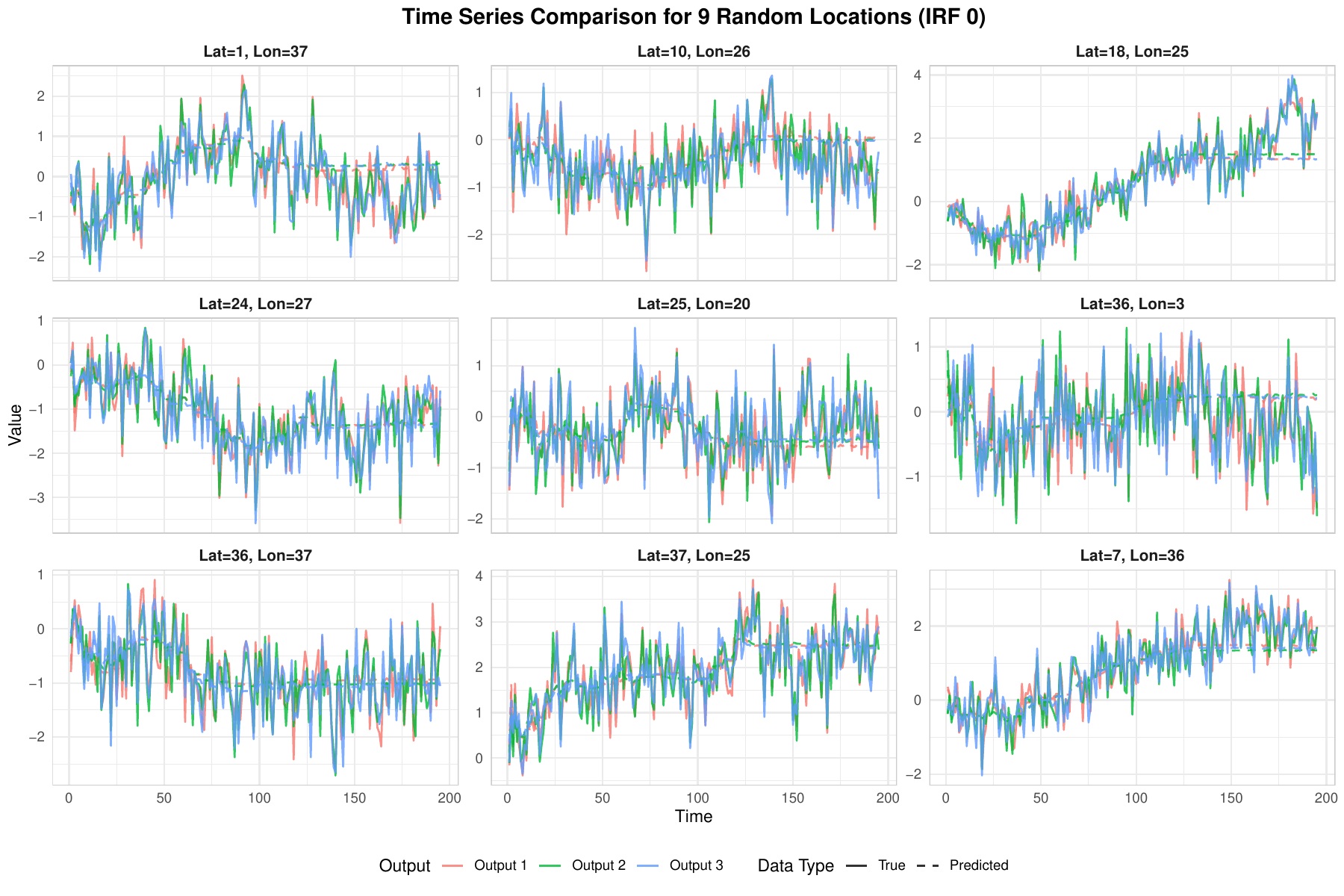}
\caption{Time series of true (solid line) and predicted (dashed line) values at nine randomly chosen grid points for IRF(0).}
\label{fig:timeseries_kappa0}
\end{figure}

\begin{figure}[H]
\centering
\includegraphics[width=0.8\linewidth]{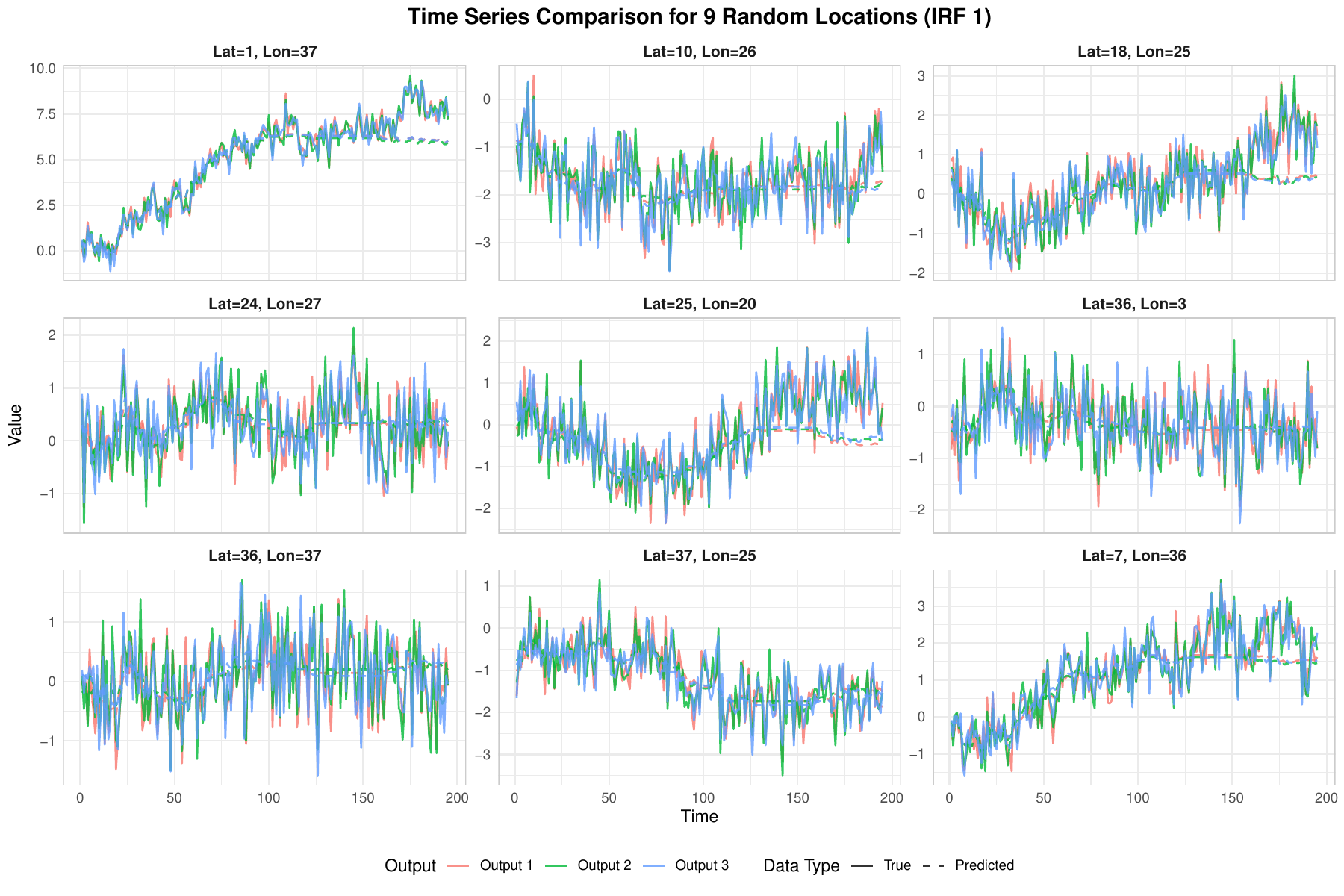}
\caption{Time series of true (solid line) and predicted (dashed line) values at nine randomly chosen grid points for IRF(1).}
\label{fig:timeseries_kappa1}
\end{figure}

\begin{figure}[H]
\centering
\includegraphics[width=0.8\linewidth]{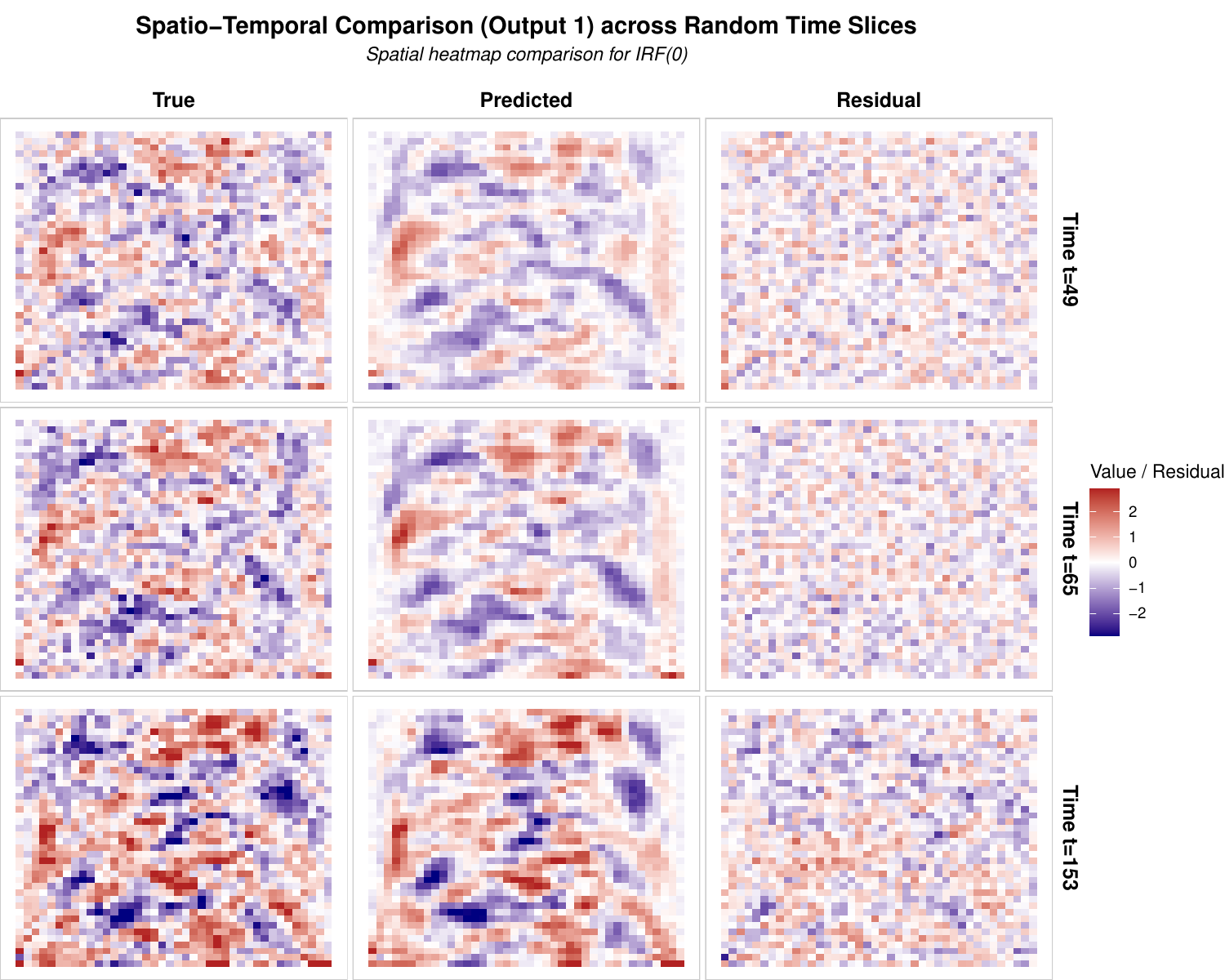}
\caption{Spatial heatmaps at randomly selected three temporal points with $Y_1$ (output 1). Left: True values; Center: Predicted Values; Right: Residuals for IRF(0)}
\label{fig:heatmap_kappa0_y1}
\end{figure}

\begin{figure}[H]
\centering
\includegraphics[width=0.8\linewidth]{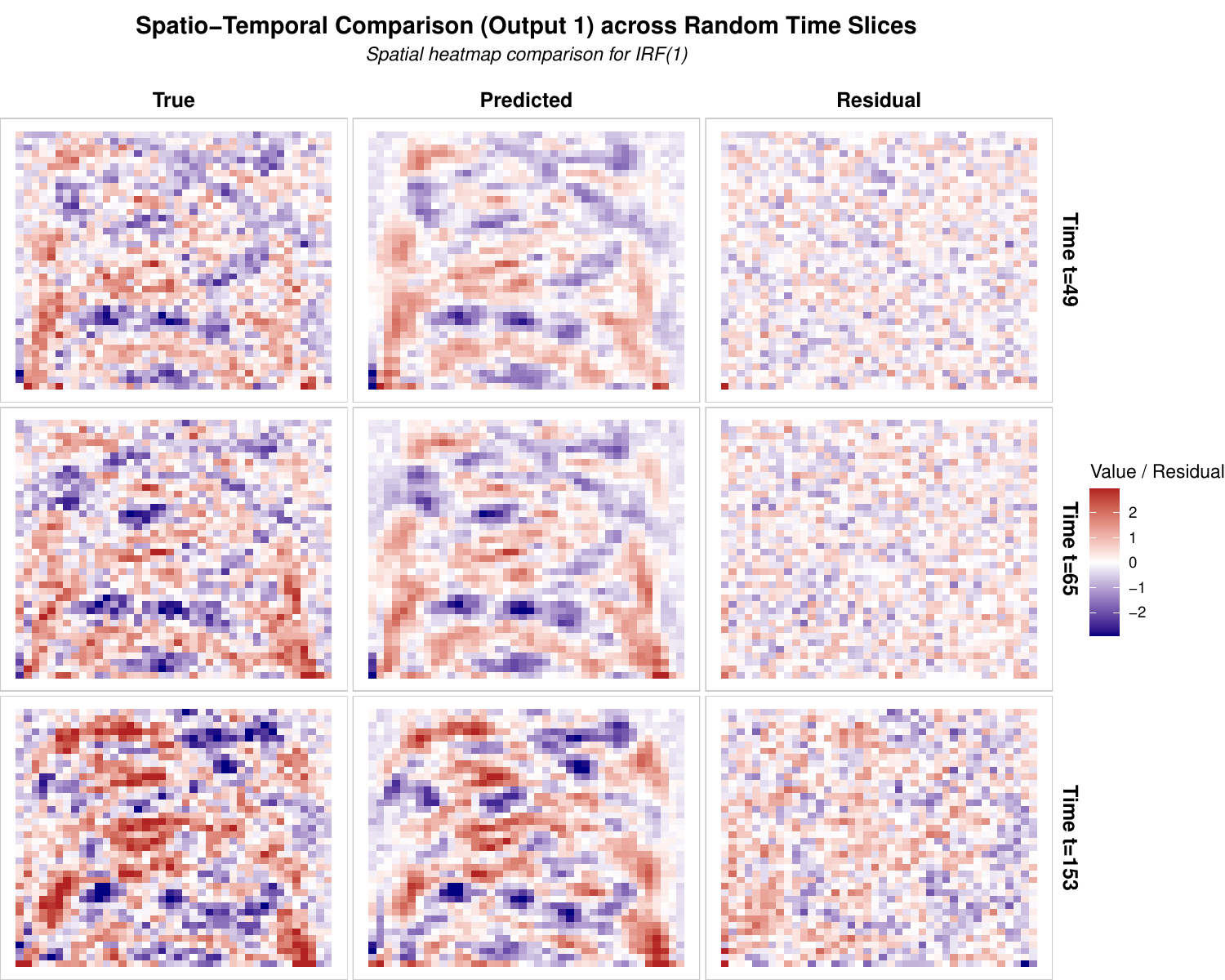}
\caption{Spatial heatmaps at randomly selected three temporal points with $Y_1$ (output 1). Left: True values; Center: Predicted Values; Right: Residuals for IRF(1)}
\label{fig:heatmap_kappa1_y1}
\end{figure}

\begin{figure}[H]
\centering
\includegraphics[width=0.8\linewidth]{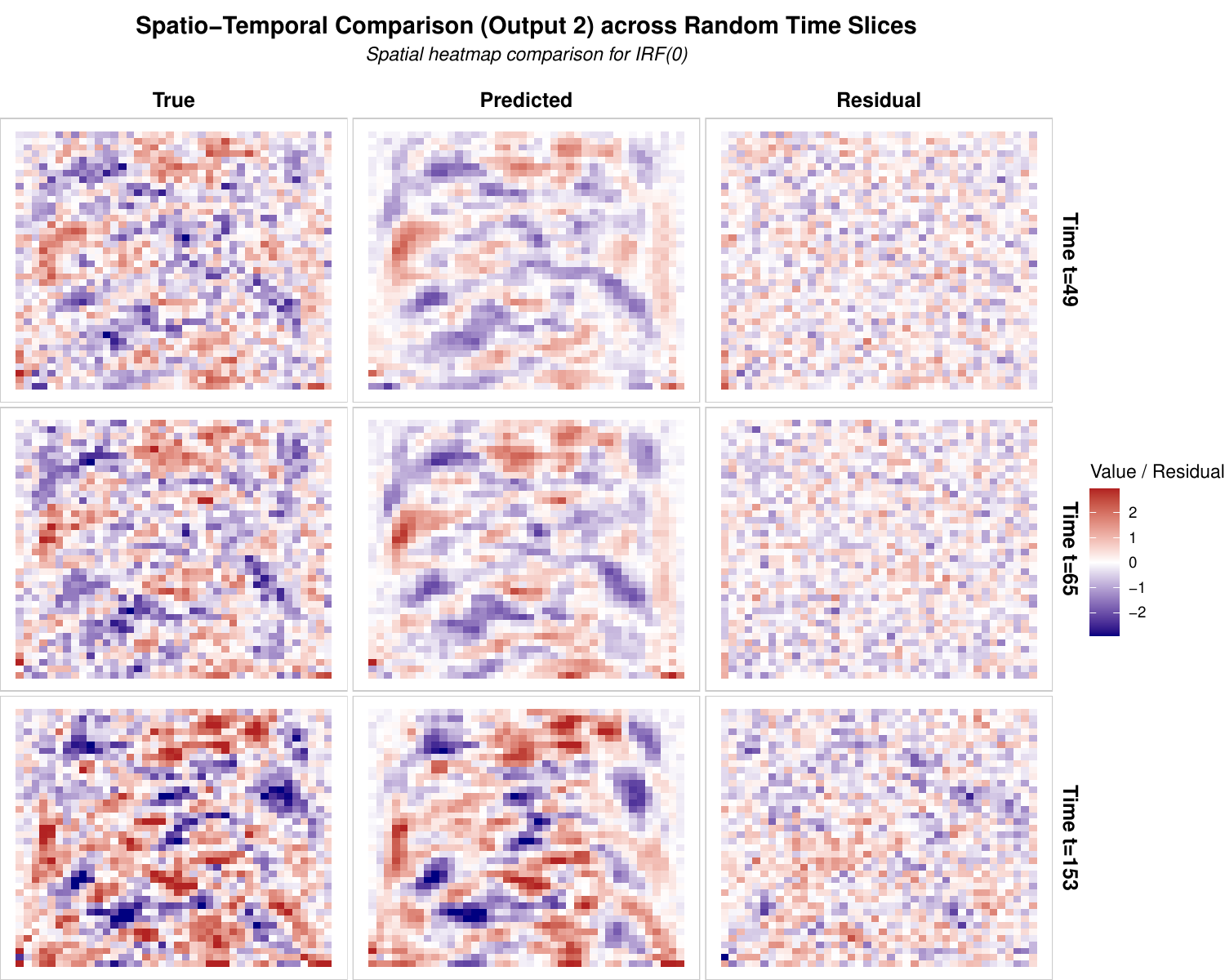}
\caption{Spatial heatmaps at randomly selected three temporal points with $Y_2$ (output 2). Left: True values; Center: Predicted Values; Right: Residuals for IRF(0)}
\label{fig:heatmap_kappa0_y2}
\end{figure}

\begin{figure}[H]
\centering
\includegraphics[width=0.8\linewidth]{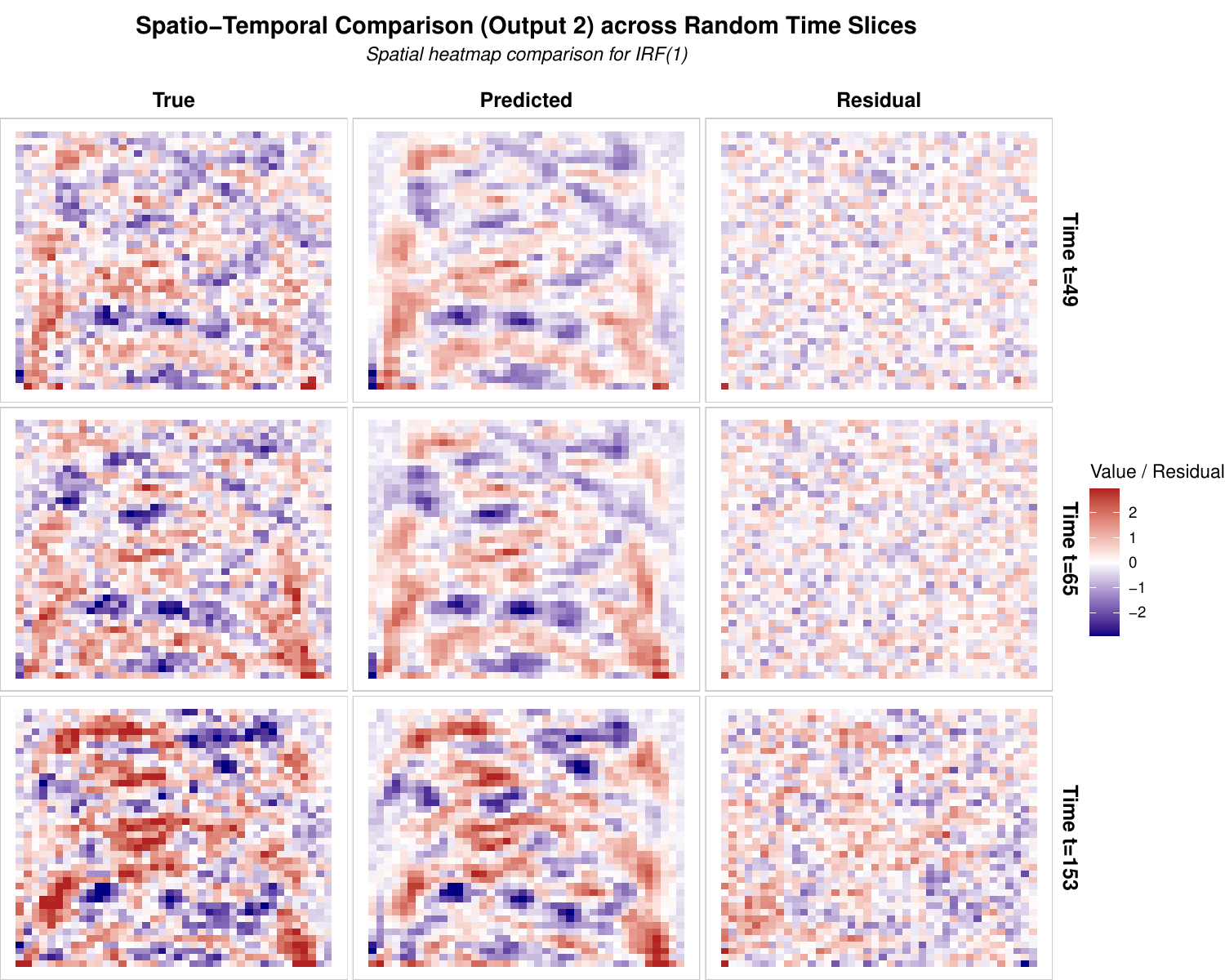}
\caption{Spatial heatmaps at randomly selected three temporal points with $Y_2$ (output 2). Left: True values; Center: Predicted Values; Right: Residuals for IRF(1)}
\label{fig:heatmap_kappa1_y2}
\end{figure}

\begin{figure}[H]
\centering
\includegraphics[width=0.8\linewidth]{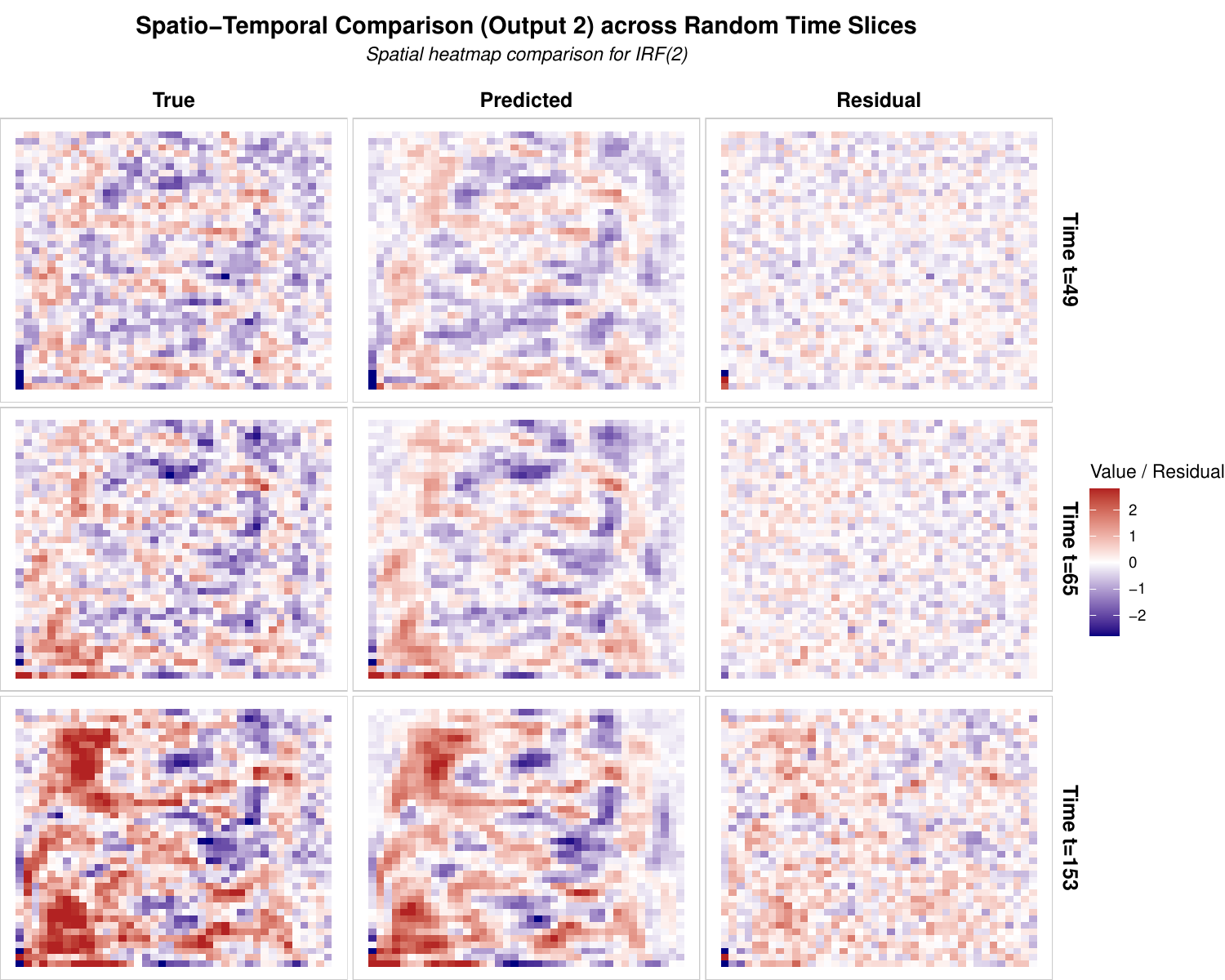}
\caption{Spatial heatmaps at randomly selected three temporal points with $Y_2$ (output 2). Left: True values; Center: Predicted Values; Right: Residuals for IRF(2)}
\label{fig:heatmap_kappa2_y2}
\end{figure}

\begin{figure}[H]
\centering
\includegraphics[width=0.8\linewidth]{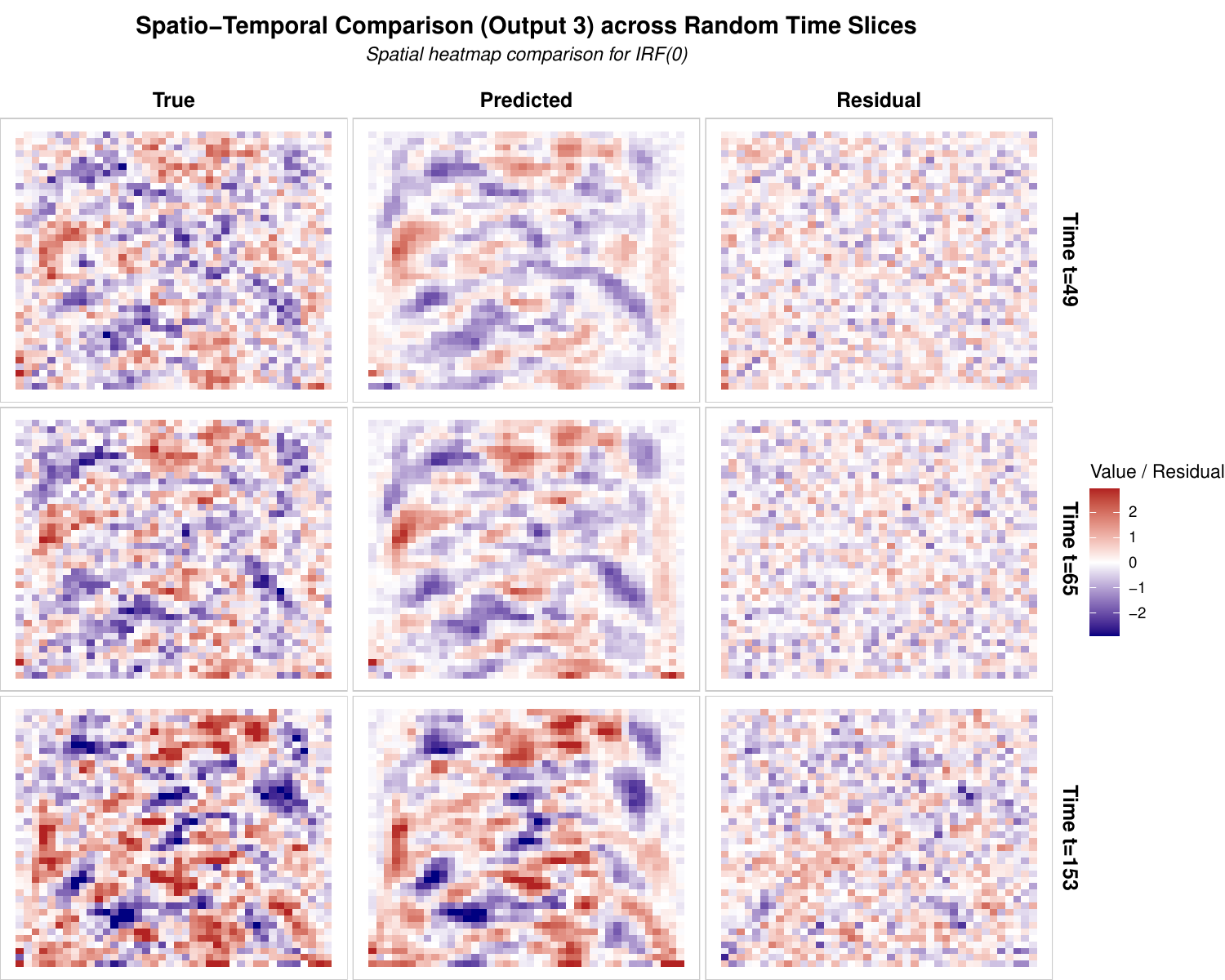}
\caption{Spatial heatmaps at randomly selected three temporal points with $Y_3$ (output 3). Left: True values; Center: Predicted Values; Right: Residuals for IRF(0)}
\label{fig:heatmap_kappa0_y3}
\end{figure}

\begin{figure}[H]
\centering
\includegraphics[width=0.8\linewidth]{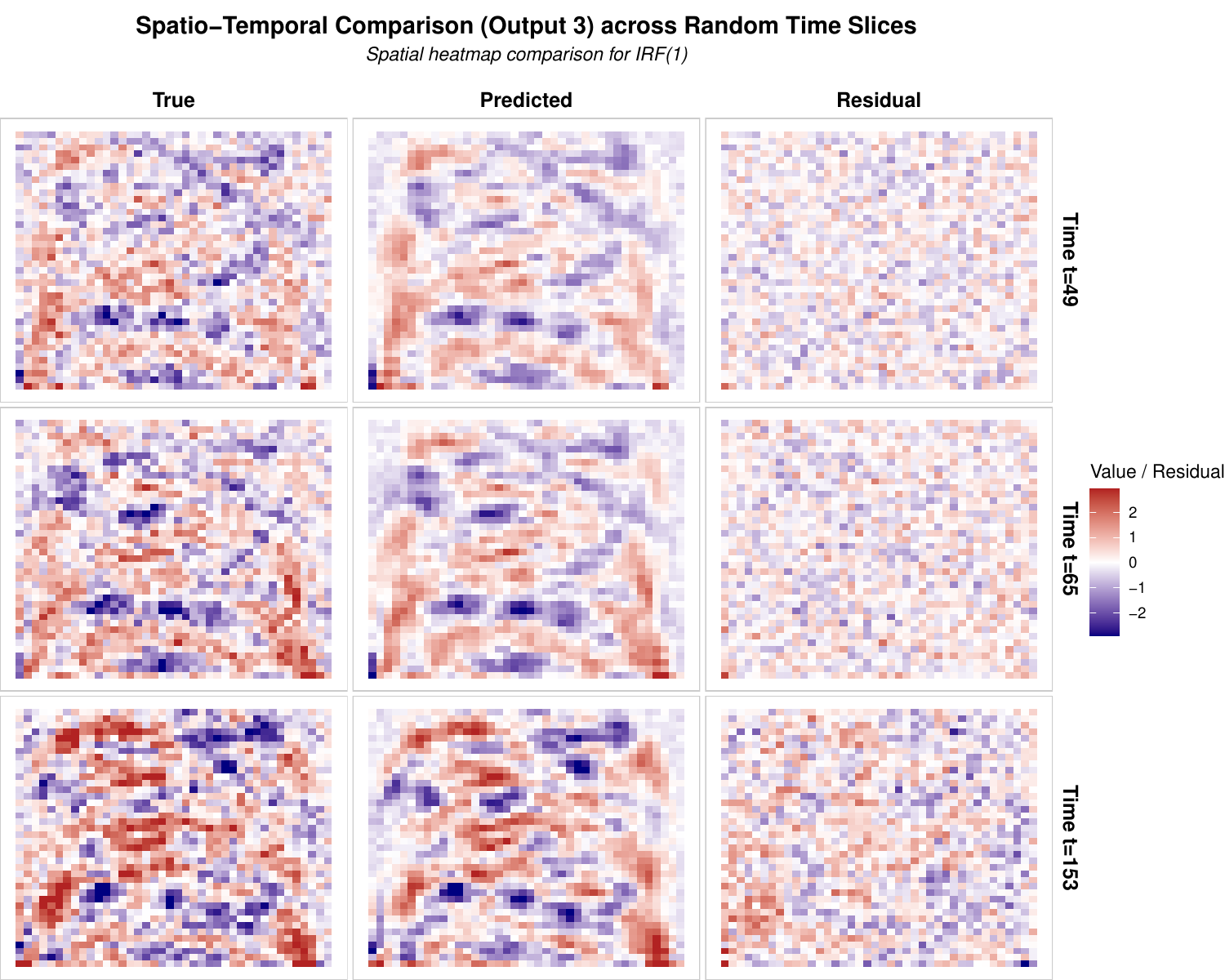}
\caption{Spatial heatmaps at randomly selected three temporal points with $Y_3$ (output 3). Left: True values; Center: Predicted Values; Right: Residuals for IRF(1)}
\label{fig:heatmap_kappa1_y3}
\end{figure}

\begin{figure}[H]
\centering
\includegraphics[width=0.8\linewidth]{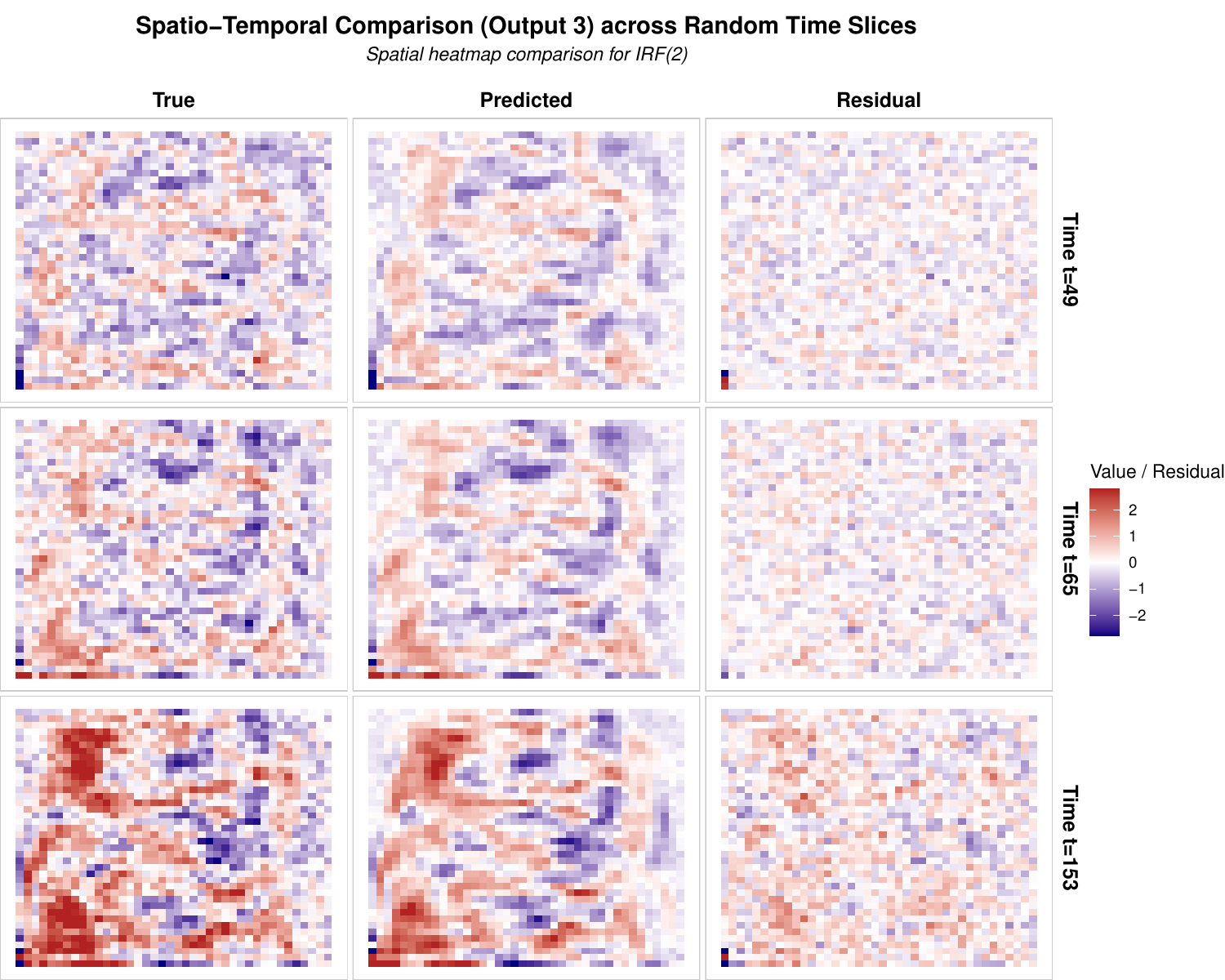}
\caption{Spatial heatmaps at randomly selected three temporal points with $Y_3$ (output 3). Left: True values; Center: Predicted Values; Right: Residuals for IRF(2)}
\label{fig:heatmap_kappa2_y3}
\end{figure}

\section{Additional Real-Data Prediction Figures}
Per-city time series, city-wise residual means, and per-variable
true-versus-predicted means for minimum and maximum temperature, complementing
the precipitation figure in the main text.

\begin{figure}[H]
\centering
\includegraphics[width=0.8\linewidth]{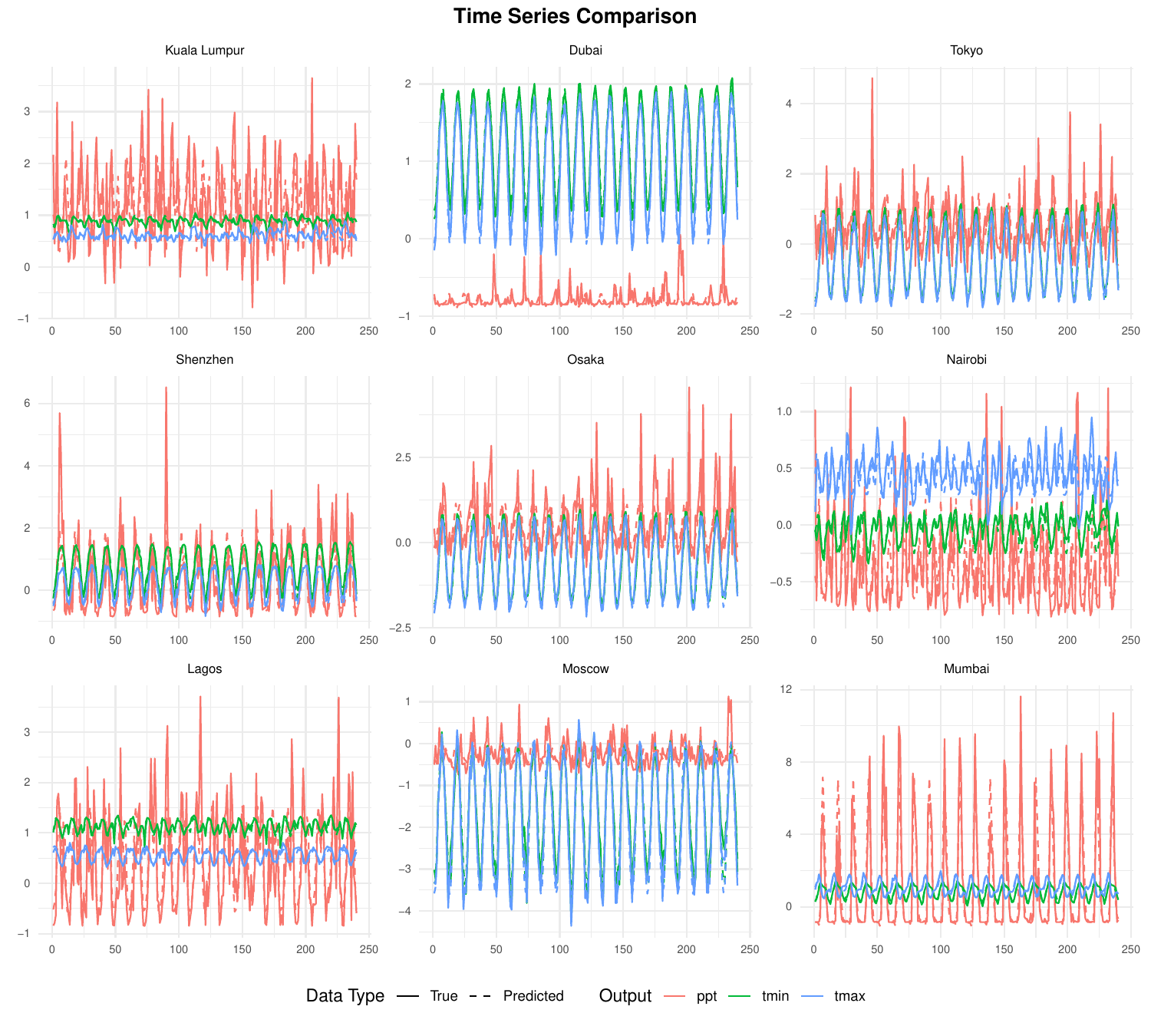}
\caption{Time series comparison for 9 randomly selected cities. Output variables are differentiated by color ($\mathtt{ppt}, \mathtt{tmin}, \mathtt{tmax}$), and data types are differentiated by linetypes (True vs. Predicted).}
\label{fig:time_series_9cities}
\end{figure}

\begin{figure}[H]
\centering
\includegraphics[width=0.8\linewidth]{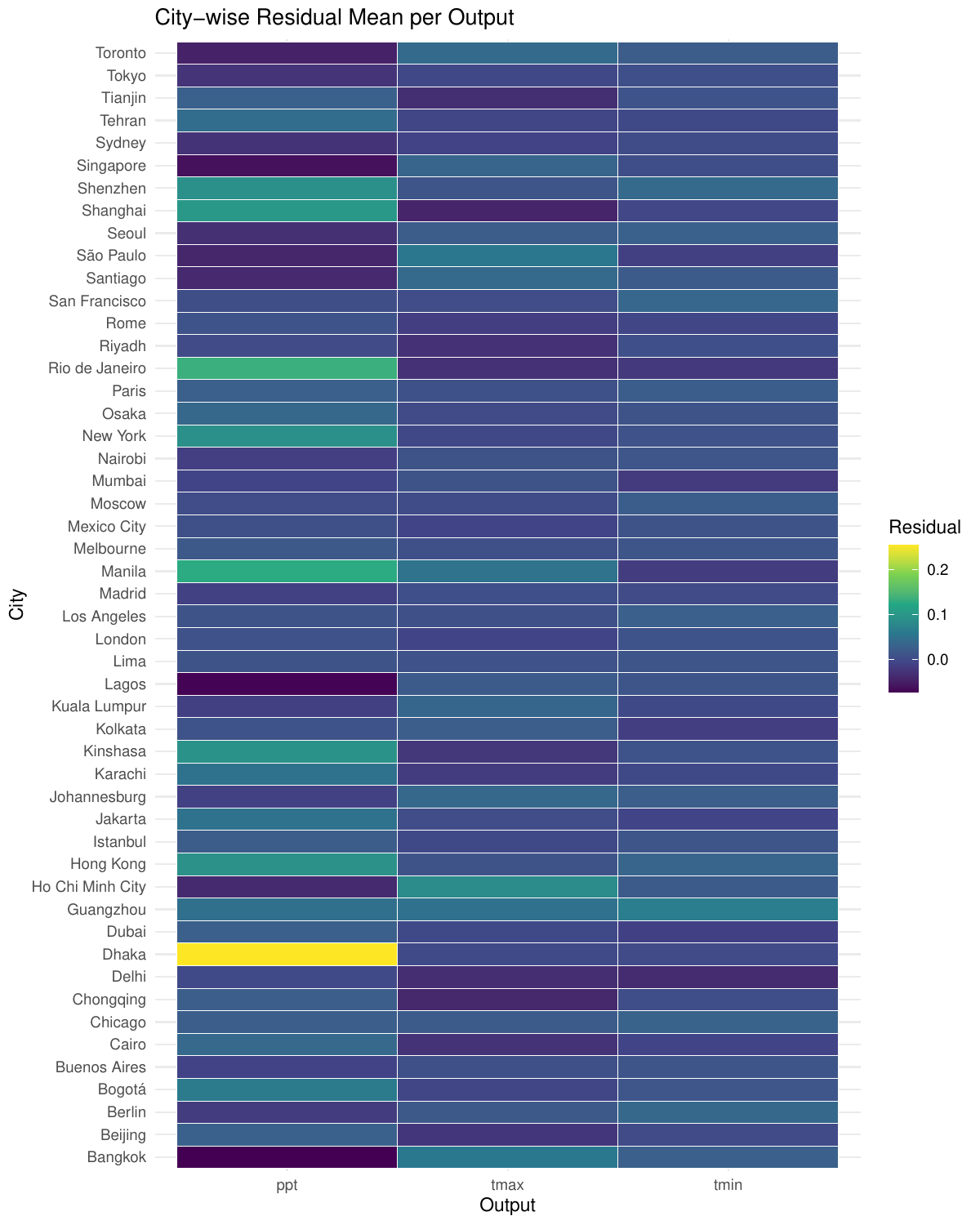}
\caption{City-wise residual mean per output (True $-$ Predicted).}
\label{fig:city_wise_residuals}
\end{figure}

\begin{figure}[H]
\centering
\includegraphics[width=0.8\linewidth]{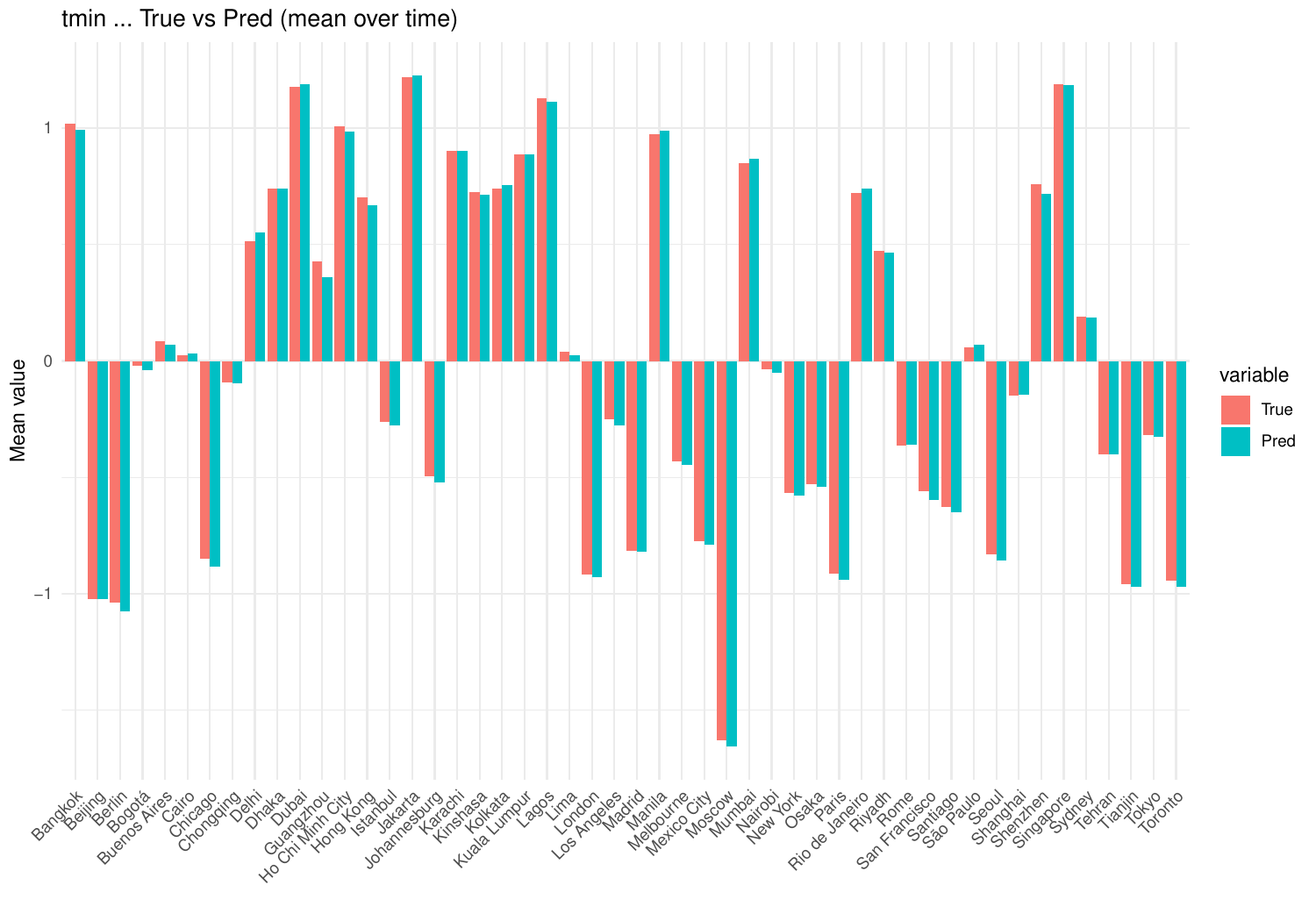}
\caption{City-wise comparison of true and predicted means (over time) for Minimum Temperature ($\mathtt{tmin}$).}
\label{fig:tmin_true_pred}
\end{figure}

\begin{figure}[H]
\centering
\includegraphics[width=0.8\linewidth]{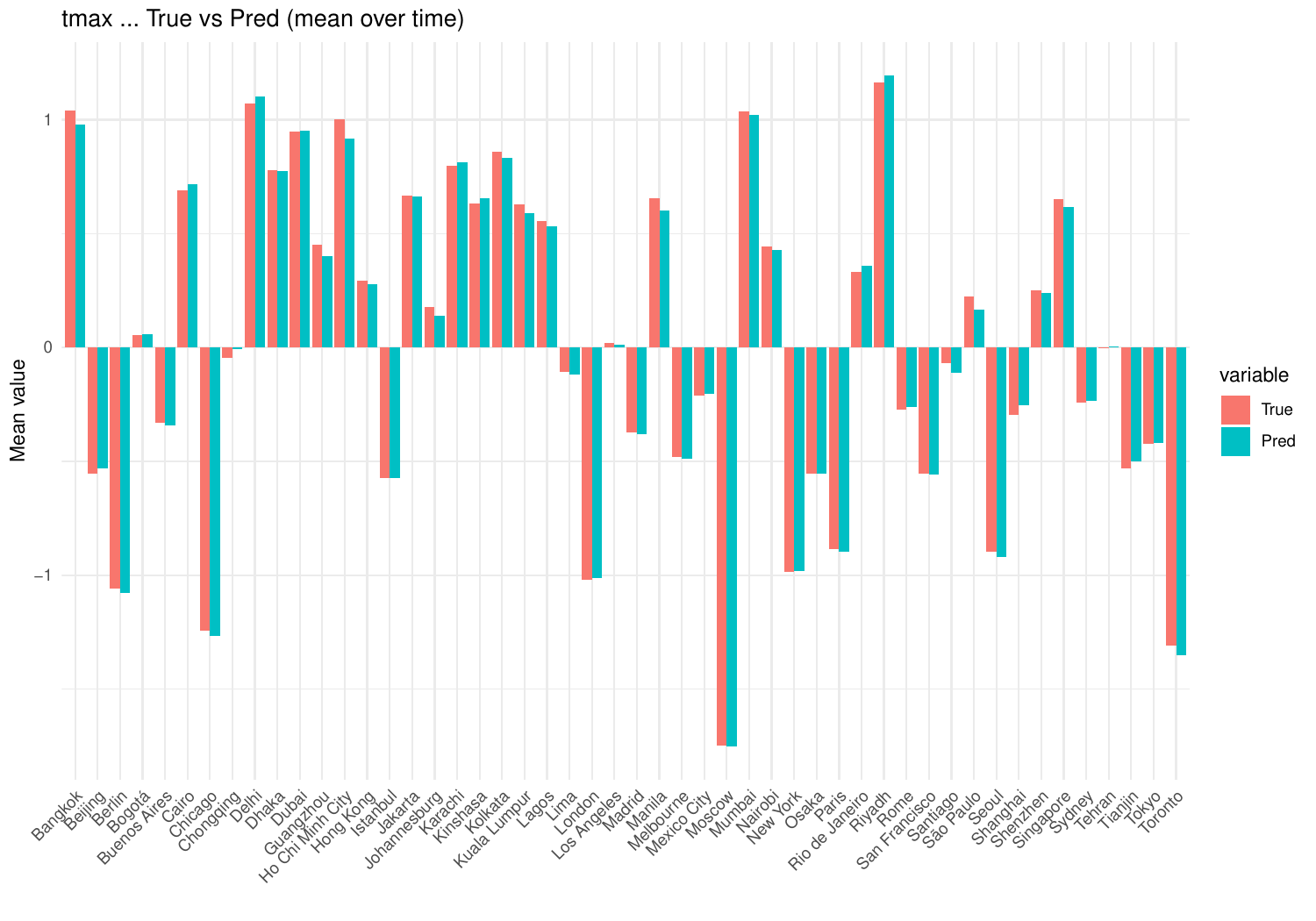}
\caption{City-wise comparison of true and predicted means (over time) for Maximum Temperature ($\mathtt{tmax}$).}
\label{fig:tmax_true_pred}
\end{figure}

\section{Derivation of the Stage-2 Precision-Weighted Backbone Gradient}
\label{app:stage2_gradient}

This appendix derives the precision-weighted backbone gradient stated in the
main text, which drives the Uncertainty-Aware Feature Refinement of Stage~2.

In Stage~2 the prediction head $W_\mu$ is strictly frozen ($\nabla_{W_\mu} = 0$),
while the shared backbone $\Phi$ continues to train under the NLL loss alongside
the variance head $W_\sigma$. Since the conditioning analysis depends only on the
joint law of the predictor--response pair and not on the temporal dependence
structure, we suppress the temporal index and write $X$ for the predictor window,
so that the two heads give $\mu(X) = W_\mu \Phi(X)$ and
$\log\sigma^2(X) = W_\sigma \Phi(X)$ (biases absorbed into the heads). The
gradient of the NLL loss with respect to the unfrozen backbone $\Phi$ is
evaluated via the multivariable chain rule:
\begin{equation}
\nabla_\Phi \mathcal{L}_{\text{NLL}} = \left(\nabla_\Phi \mu \right) \frac{\partial \mathcal{L}_{\text{NLL}}}{\partial \mu} + \left(\nabla_\Phi \log\sigma^2 \right) \frac{\partial \mathcal{L}_{\text{NLL}}}{\partial \log\sigma^2}
\end{equation}
To expose the precision-weighting mechanism, we evaluate the partial derivatives
of the marginal Gaussian component $\frac{1}{2}\log\sigma^2 + \frac{(Y-\mu)^2}{2\sigma^2}$
(the copula penalty also contributes to the backbone gradient through
$\mathbf{w}=(Y-\hat\mu)/\hat\sigma$, but does not alter the precision-weighting
structure that is our focus here):
\begin{equation}
\frac{\partial \mathcal{L}_{\text{NLL}}}{\partial \mu} = -\frac{Y-\mu}{\sigma^2}, \quad \frac{\partial \mathcal{L}_{\text{NLL}}}{\partial \log\sigma^2} = \frac{1}{2} - \frac{(Y-\mu)^2}{2\sigma^2}
\end{equation}
Since $\mu$ is a linear combination of features $\Phi$, the Jacobian is simply
the transpose of the weight matrix ($\nabla_\Phi \mu = W_\mu^\top$ and
$\nabla_\Phi \log\sigma^2 = W_\sigma^\top$). Substituting these yields the
precision-weighted gradient component applied to the backbone:
\begin{equation} \label{eq:app_precision_weight}
\nabla_\Phi \mathcal{L}_{\text{NLL}} = \underbrace{-W_\mu^\top \left( \frac{Y - \mu(X)}{\sigma^2(X)} \right)}_{\text{Precision-Weighted Signal Learning}} + \underbrace{W_\sigma^\top \left( \frac{1}{2} - \frac{(Y - \mu(X))^2}{2\sigma^2(X)} \right)}_{\text{Uncertainty Calibration}}
\end{equation}
which is the expression interpreted in the main text.

\section{Definitions of the Per-Variable RMSE and the Coverage Metric}
\label{app:metric_defs}

This appendix states in full the two evaluation metrics whose definitions are
summarized in Section~4.2 of the main text. Throughout, $i$ indexes the $N$
spatial locations, $t$ ranges over the held-out test snapshots
$\mathcal{T}_{\text{test}}$, and $k$ indexes the output variable; $Y_{i,t,k}$ is
the observed value and $(\hat{\mu}_{i,t,k}, \hat{\sigma}_{i,t,k})$ the predicted
mean and standard deviation.

The per-output-variable root-mean-square error is
\begin{equation}
\text{RMSE}_k = \sqrt{\frac{1}{N\,|\mathcal{T}_{\text{test}}|} \sum_{i=1}^{N} \sum_{t \in \mathcal{T}_{\text{test}}} (Y_{i,t,k} - \hat{\mu}_{i,t,k})^2}, \quad k=1,\dots,n_\text{out},
\end{equation}
which aggregates to the overall RMSE reported in the main text by additionally
averaging the squared errors over the $n_{\text{out}}$ outputs before taking the
square root.

The 95\% Prediction Interval Coverage Probability is
\begin{equation}
\text{PICP}_k = \frac{1}{N\,|\mathcal{T}_{\text{test}}|} \sum_{i=1}^{N} \sum_{t \in \mathcal{T}_{\text{test}}} \mathbf{1} \left( Y_{i,t,k} \in \left[ \hat{\mu}_{i,t,k} - 1.96 \hat{\sigma}_{i,t,k}, \, \hat{\mu}_{i,t,k} + 1.96 \hat{\sigma}_{i,t,k} \right] \right),
\end{equation}
the empirical probability that the observed value falls within the nominal
$95\%$ predictive interval. Both metrics are computed on the held-out test
snapshots only.

\section{Predictive Intervals of the Kriging Baseline}\label{app:kriging_pi}

Section~4.3 of the main text compares the proposed model's predictive intervals with the
kriging baseline's on two axes that its marginal coverage does not capture. This appendix
gives both comparisons in full. Kriging's own PICP and NLL are in Table~1 of the main text.

\subsection{Adaptivity}

The kriging variance is a function of the design geometry and the fitted variogram and not
of the observed values, so it is constant in time by construction; the residual spatial
variation in Table~\ref{tab:pi_adaptivity} below reflects only the irregularity of the
sampled prediction locations.

\begin{table}[!ht]
\centering
\caption{Adaptivity of the predictive standard deviation, measured as its coefficient of variation across locations ($\mathrm{cv}_{\text{sp}}$) and across time ($\mathrm{cv}_{\text{tm}}$). Simulation entries are means over 50 Monte Carlo runs at $\eta=10$; the real-data entries are means over the 50 model-initialization seeds.}
\label{tab:pi_adaptivity}
\begin{tabular}{l c c c c}
\toprule
& \multicolumn{2}{c}{\textbf{Proposed (Two-stage)}} & \multicolumn{2}{c}{\textbf{Kriging Baseline}} \\
\cmidrule(lr){2-3}\cmidrule(lr){4-5}
\textbf{Setting} & $\mathrm{cv}_{\text{sp}}$ & $\mathrm{cv}_{\text{tm}}$ & $\mathrm{cv}_{\text{sp}}$ & $\mathrm{cv}_{\text{tm}}$ \\
\midrule
IRF(0), $\eta=10$ & $0.816$ & $0.025$ & $0.020$ & $0.0000$ \\
IRF(1), $\eta=10$ & $0.795$ & $0.036$ & $0.020$ & $0.0000$ \\
IRF(2), $\eta=10$ & $0.212$ & $0.016$ & $0.020$ & $0.0000$ \\
\addlinespace[3pt]
Real climate data & $1.321$ & $0.469$ & $0.030$ & $0.0000$ \\
\bottomrule
\end{tabular}
\end{table}

\subsection{The joint predictive law}

The baseline is three independent per-variable krigings, so it supplies marginal intervals
only. For a Gaussian predictive law the cost of ignoring a correlation matrix is
$-\tfrac{1}{2}\log|\mathbf{R}|$ nats per space--time point. We evaluate this on the shrunk
correlation $\mathbf{R}_s = (1-\alpha)\hat{\mathbf{R}} + \alpha\mathbf{I}$ that the
objective of Section~3.3.1 of the main text actually uses, rather than on the raw
$\hat{\mathbf{R}}$, and charge the model for any mismatch between $\mathbf{R}_s$ and the
residual correlation $\mathbf{W}$ it attains, so that the quantity reported is the realized
gain
\[
  \tfrac{1}{2}\big[\operatorname{tr}(\mathbf{W}) - \log|\mathbf{R}_s| - \operatorname{tr}(\mathbf{R}_s^{-1}\mathbf{W})\big]
\]
rather than an oracle bound. The gain is computed per run; substituting the mean
$\hat{\mathbf{R}}$ across runs would conceal the boundary saturation discussed in
Section~6 of the main text.

\begin{table}[!ht]
\centering
\caption{Joint-likelihood gain from modelling cross-variable dependence, against the marginal-NLL deficit to the kriging baseline under the MLE-temperature scale (mean $\pm$ standard deviation over 50 Monte Carlo runs). A negative net figure means the proposed model is ahead on the joint predictive law even where it trails kriging on the marginal one.}
\label{tab:joint_nll}
\begin{tabular}{l c c c c}
\toprule
\textbf{Setting} & \textbf{Realized gain} & \textbf{Median} & \textbf{Marginal deficit} & \textbf{Net} \\
\midrule
IRF(0), $\eta=10$ & $1.055 \pm 0.062$ & $1.041$ & $0.235$ & $-0.820$ \\
IRF(1), $\eta=10$ & $1.047 \pm 0.050$ & $1.035$ & $0.254$ & $-0.793$ \\
IRF(2), $\eta=10$ & $2.679 \pm 2.126$ & $3.194$ & $2.172$ & $-0.507$ \\
\bottomrule
\end{tabular}
\end{table}

The gain is positive in $50$, $50$ and $49$ of the $50$ runs respectively. At IRF(2) it is
both larger and far more variable, which is where the fixed shrinkage $\alpha$ is doing the
most work; Section~6 of the main text takes this up.

\end{document}